%% file: iclr2027_conference.tex
\documentclass{article} %
\usepackage{iclr2027_conference,times}

\input{math_commands.tex}

\usepackage[utf8]{inputenc}
\usepackage[T1]{fontenc}
\usepackage{hyperref}
\usepackage{url}
\usepackage{booktabs}
\usepackage{multirow}
\usepackage{amsmath,amssymb,amsfonts,mathtools}
\usepackage{microtype}
\usepackage{xcolor}
\usepackage{graphicx}
\usepackage{tabularx}
\usepackage{algorithm}
\usepackage{algpseudocode}
\usepackage{cleveref}
\crefname{algorithm}{Algorithm}{Algorithms}
\Crefname{algorithm}{Algorithm}{Algorithms}
\usepackage{enumitem}
\usepackage{placeins}
\usepackage{amsthm}
\usepackage{amsmath}
\usepackage{amssymb}
\usepackage{nicefrac}

\usepackage{colortbl}
\definecolor{bestbg}{RGB}{198,230,210}
\definecolor{secondbg}{RGB}{211,225,242}

\definecolor{besttext}{RGB}{45,130,75}
\definecolor{secondtext}{RGB}{60,105,160}

\newtheorem{observation}{Observation}
\newtheorem{proposition}{Proposition}

\usepackage{tikz}
\usetikzlibrary{arrows.meta,positioning,fit,calc}

\newcommand{\I}{\mathbf{I}}
\newcommand{\Zero}{\mathbf{0}}
\newcommand{\m}[1]{\mathbf{#1}}

\newcommand{\diag}{\operatorname{diag}}
\newcommand{\GELU}{\operatorname{GELU}}
\newcommand{\LN}{\operatorname{LN}}
\newcommand{\Drop}{\operatorname{Drop}}
\newcommand{\GL}{\operatorname{GL}}

\newcommand{\dt}{\Delta t}
\newcommand{\FFN}{\operatorname{FFN}}

\definecolor{impgreen}{RGB}{0,132,61}
\newcommand{\impc}[1]{\textcolor{impgreen}{\textbf{#1}}}
\newcommand{\noimp}[1]{\textcolor{black!50}{#1}}

\title{WaveGSSM: Graph Wave State Space Models for Propagating Spatio-Temporal Patterns}

\author{
\textbf{Junyou Zhu}$^{1,2}$ \quad
\textbf{Fenying Cai}$^{1}$ \quad
\textbf{Ping Xiong}$^{2}$ \quad
\textbf{Christian Nauck}$^{1}$ \quad
\textbf{Langzhou He}$^{3}$ \\
\textbf{Chao Gao}$^{4}$ \quad
\textbf{Juergen Kurths}$^{1,5}$ \quad
\textbf{Frank Hellmann}$^{1}$
\\[2mm]
$^{1}$Potsdam Institute for Climate Impact Research, Germany \\
$^{2}$Machine Learning Group, Technical University of Berlin, Germany \\
$^{3}$Department of Computer Science, University of Illinois at Chicago, USA \\
$^{4}$Northwestern Polytechnical University, China \\
$^{5}$Department of Physics, Humboldt University Berlin, Germany
}

\iclrfinalcopy %

\begin{document}

\maketitle
\lhead{Preprint}

\begin{abstract}
Spatio-temporal graph models typically encode each snapshot with a GNN and then connect the resulting representations through a temporal module. This space-then-time design is effective, yet it does not explicitly represent how a pattern moves across the graph. We show empirically that, for a propagating process, the same present field can lead to different futures when its recent rate of change differs, motivating an explicit representation of motion in the predictive state. We introduce WaveGSSM, a second-order graph state-space model that maintains two coupled latent states at each node, one for the current pattern and one for its temporal rate of change. A graph-wave transition updates the motion state through graph interactions and uses it to advance the pattern state, coupling spatial propagation and temporal evolution within a single rollout. We evaluate WaveGSSM on four temporal-graph benchmarks and global weather forecasting. It consistently achieves the best mean performance across the temporal-graph benchmarks and reduces the geopotential RMSE by $20.2\%$ on average for 1- to 5-day weather forecasts relative to a backbone-matched snapshot model, while better preserving large-scale atmospheric patterns.
\end{abstract}

\section{Introduction}
\label{sec_introduction}
Spatio-temporal graph learning commonly combines graph-based spatial encoding with recurrence \citep{cini2023taming}, attention \citep{karmim2024supra}, or state-space models \citep{li2024graphssm} for temporal modeling. Most methods therefore first encode each graph snapshot and then connect the resulting representations across time \citep{wang2025simple}. However, this space-then-time factorization does not explicitly represent how a pattern moves through the graph. For propagating processes, the same spatial distribution of activation can lead to different futures when its rate of change, or velocity, differs. Existing models may infer this velocity from temporal history, but it remains implicit rather than being represented in the graph state. Errors in the inferred velocity can accumulate over recursive forecasts and shift the predicted pattern away from its correct location. This raises a natural question. Can the activation pattern and its velocity be represented jointly, so that propagation through graph edges and evolution through time are advanced within the same transition?

\begin{figure}[h]
\centering
\includegraphics[width=\linewidth]{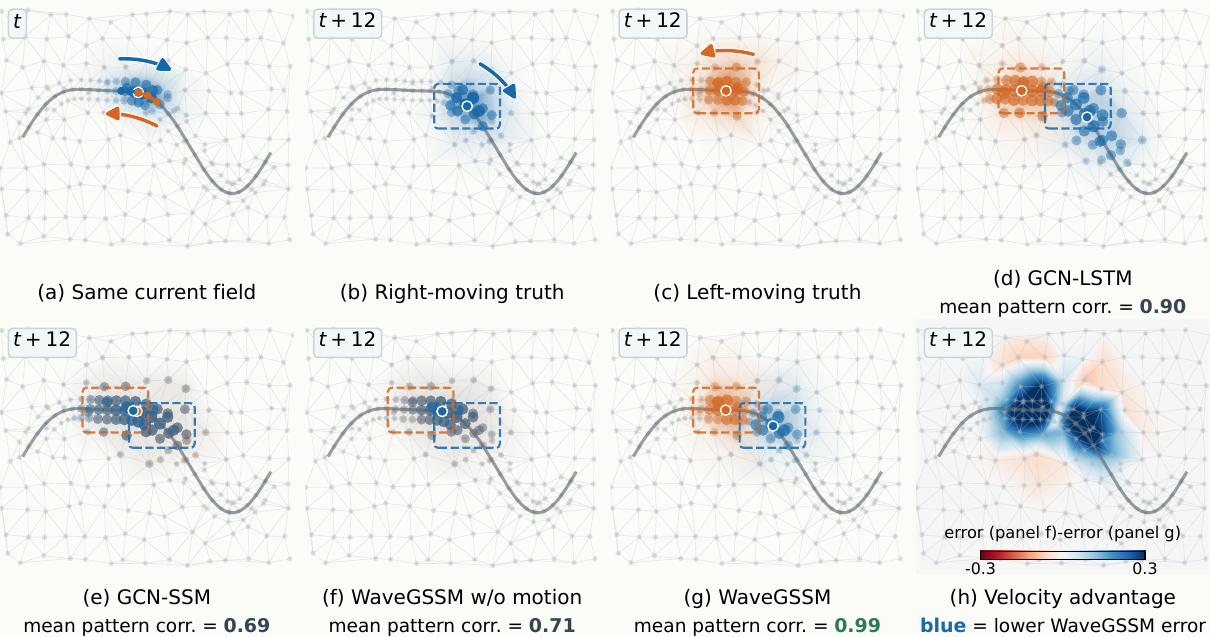}
\caption{A synthetic experiment shows that the present graph field alone does not determine the future of a propagating pattern. Paired trajectories, shown in blue and orange, share the same graph and present field at time $t$ but arrive with opposite recent motion (a), leading to different future locations at $t+12$ (b, c). Panels (d)--(g) compare history-aware baselines, WaveGSSM without the motion state, and the full WaveGSSM. The values below panels (d)--(g) report the mean Pearson pattern correlation over the two motion directions at $t+12$ for the displayed held-out pair. Panel (h) visualizes the error reduction from carrying the motion state. WaveGSSM best preserves the two motion-dependent futures, while aggregate future-separation and centroid-error results are reported in Appendix~\Cref{fig:intro_rollout_aggregate}.}
\label{fig_intro_rollout}
\end{figure}

Answering this question requires a predictive representation that captures both the present pattern and how it is moving. \Cref{fig_intro_rollout} shows why this information matters in a simple synthetic forecasting setting. At time $t$, the two trajectories in panel~(a) have the same graph and identical node patterns, but their recent histories indicate motion in opposite directions. They therefore evolve to different future locations, as shown in panels~(b) and~(c). For each trajectory, all history-aware models receive the same three-frame history ending at time $t$. We ask whether they can use this history to predict the two different futures. We compare two space-then-time baselines, GCN-LSTM and GCN-SSM, against WaveGSSM and a variant without the motion state. GCN-LSTM and GCN-SSM combine snapshot-level graph encoding with a temporal module. They capture part of the motion information, while the full WaveGSSM most accurately predicts both motion-dependent futures. This motivates representing pattern and motion jointly within the evolving graph state rather than leaving motion implicit in temporal memory. We also provide a theoretical analysis of this state requirement in Observation~\ref{thm:velocity_phase}.

The same state-representation problem also appears in real medium-range weather forecasting, where coherent atmospheric patterns move over several days while changing their strength and shape~\citep{kornhuber2020amplified,wang2025evidence}. A forecast can therefore produce a plausible large-scale pattern while placing it at the wrong location. \Cref{fig_intro_atmosphere} illustrates this failure mode in a five-day forecast of 500\,hPa geopotential height, or z500, a standard field for describing large-scale atmospheric circulation. A backbone-matched GraphCast-style predictor~\citep{lam2023graphcast} captures the broad Atlantic-to-Europe pattern but displaces the ridge--trough structure relative to ERA5, while WaveGSSM keeps the predicted structure more closely aligned. Hence, this example further motivates modeling not only the present pattern, but also its motion.

\begin{figure}[th]
\centering
\includegraphics[width=\linewidth]{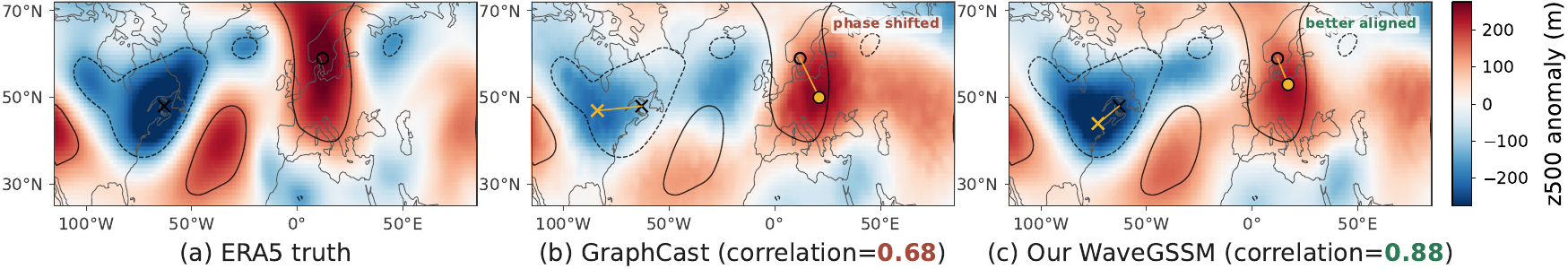}
\caption{A plausible weather forecast can still place a coherent pattern at the wrong location. (a) The heatmap shows the ERA5 reanalysis as the evaluation ground truth of predicting five-day 500\,hPa geopotential-height anomaly, or z500. (b) Prediction by GraphCast. (c) Prediction by our WaveGSSM. Yellow markers show the predicted centers and displacement to the ground-truth (marked with black contours and markers). GraphCast$^{*}$ captures the broad pattern but shifts it relative to ERA5, whereas WaveGSSM keeps the structure better aligned, with correlation increasing from $0.68$ to $0.88$.}

\label{fig_intro_atmosphere}
\end{figure}

We therefore introduce WaveGSSM, a wave-inspired Graph State Space Model for propagating spatio-temporal signals. Following second-order wave dynamics, each node maintains two coupled latent states. The pattern state represents the current latent pattern, while the motion state represents its temporal rate of change. A graph-wave transition updates the motion state through graph interactions and then uses the updated motion to advance the pattern state. Spatial propagation and temporal evolution are therefore coupled within a single state transition. WaveGSSM explicitly incorporates motion into the evolving graph state rather than leaving it to be inferred implicitly by a separate temporal module.

\textbf{Our main contributions can be summarized as follows:}
\begin{itemize}
\item We show that the present field can be an incomplete predictive state for propagating graph signals. Paired synthetic trajectories and a theoretical modal analysis show how different recent motion can produce different futures from the same present field, motivating the explicit representation of motion in the graph state.
\item We introduce WaveGSSM, a wave-inspired graph state-space model with coupled pattern and motion states. Its second-order graph-wave transition advances the two states jointly, coupling spatial propagation and temporal evolution in a single transition.
\item We demonstrate strong performance across four temporal-graph benchmarks and global weather forecasting. WaveGSSM achieves the best mean performance on all four temporal-graph benchmarks and reduces the geopotential RMSE by $20.2\%$ on average for 1- to 5-day forecasts relative to a backbone-matched snapshot model, with larger pattern gains when propagation becomes more important.
\end{itemize}

\section{Related Work}
\label{sec:related}

\paragraph{Spatio-temporal graph learning.}
Many spatio-temporal graph models factor spatial and temporal modeling by first encoding each graph snapshot and then connecting the resulting representations through recurrence~\cite{STGCN2018,DGCN2023,cini2023taming}, attention~\cite{karmim2024supra}, or state-space sequence models \citep{li2024graphssm}. Recent work has also explored generative pre-training \citep{li2023gptst}, continuous operators on product graphs \citep{einizade2024citrus}, and continual adaptation to streaming graphs \citep{chen2025expand}. Across these formulations, motion is typically inferred from successive representations rather than represented as part of the evolving graph state. WaveGSSM instead couples the current latent pattern and its recent motion within a single dynamical state.

\paragraph{Graph state-space models.}
GraphSSM is the closest prior work in this family \citep{li2024graphssm}. It brings state-space sequence modeling to temporal graphs by incorporating evolving graph structure into the temporal processor. Related approaches include GraphMamba for graph representation learning \citep{behrouz2024graphmamba} and message-passing state-space models that combine sequence modeling with graph propagation \citep{ceni2025mpssm}. Our new WaveGSSM differs in the structure of the state itself. Its latent state explicitly contains coupled pattern and motion states whose joint evolution determines the next graph transition.

\paragraph{Physics-inspired graph dynamics.}
Differential equations have been used to structure graph representation learning. GraphCON builds GNNs from graph-coupled oscillator dynamics \citep{rusch2022graphcon}, Graph Wave Networks formulates message passing as wave propagation \citep{yue2025graphwave}, and TDE-GNN models temporal dependencies with differential-equation-inspired graph dynamics \citep{eliasof2024temporal}. However, these approaches mainly use the dynamics to organize representation propagation within the graph network, rather than to model the motion state of an observed temporal graph process. In this paper, WaveGSSM instead uses graph-wave dynamics as the recurrent state transition along the observed sequence to solve the spatio-temporal pattern learning problem. Specifically, its pattern and motion states jointly encode the current latent pattern and its temporal change, allowing spatial propagation and temporal evolution to be advanced in the same transition.
\begin{figure}[th]
\centering
\includegraphics[width=\linewidth]{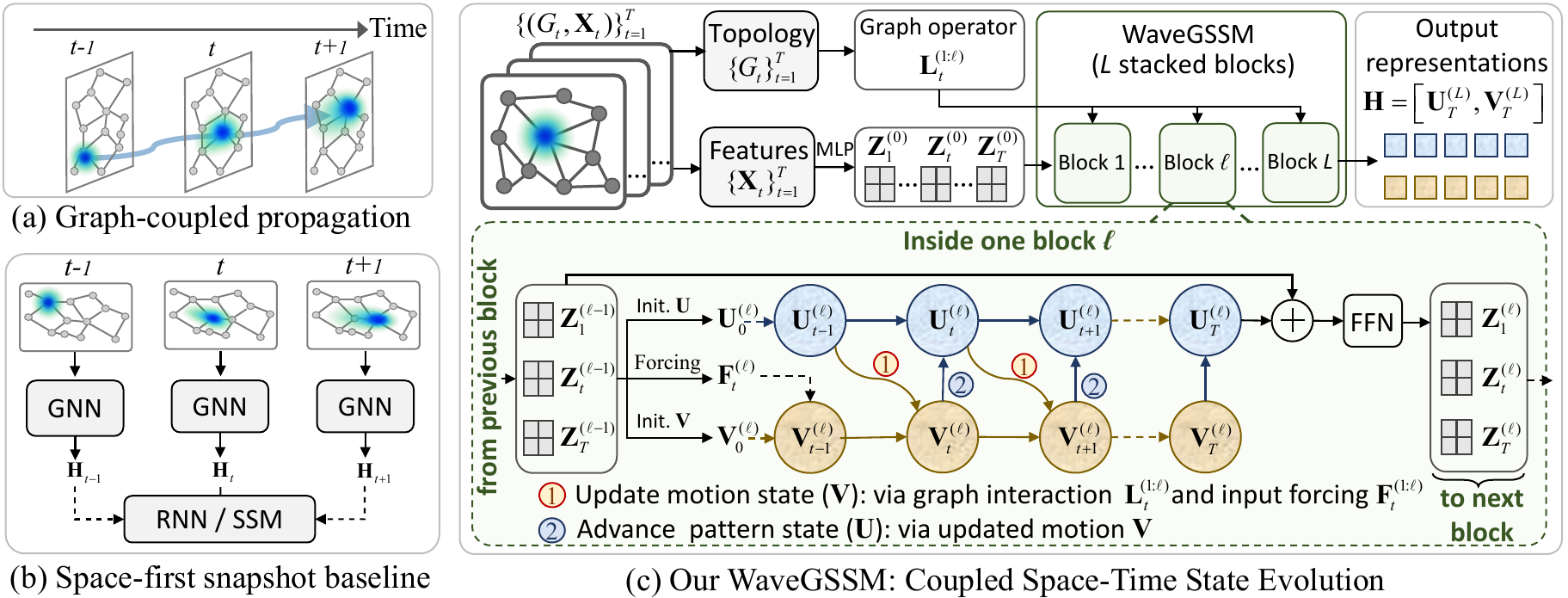}
\caption{Motivation and architecture of WaveGSSM. (a) Spatio-temporal graph signals contain patterns that propagate through graph edges over time rather than evolving as independent snapshots. (b) Snapshot baselines first encode each snapshot with a GNN and then connect the representations with a temporal module such as an RNN or SSM, yielding a space-then-time pipeline. (c) WaveGSSM instead rolls forward coupled pattern and motion states within each graph-wave block. The pattern state $\mathbf{U}_t$ represents the current latent pattern, while the motion state $\mathbf{V}_t$ represents its temporal rate of change. Graph interaction and an input-dependent forcing update $\mathbf{V}_t$, which then advances $\mathbf{U}_t$, coupling spatial propagation and temporal evolution within a single rollout.}
\label{fig_method_frame}
\end{figure}

\section{Method}
\label{sec:method}

We consider an observed temporal graph sequence $\{(G_t,\m X_t)\}_{t=1}^{T}$, where $G_t=(\mathcal V,E_t)$, $N=|\mathcal V|$, and $\m X_t\in\R^{N\times F_{\mathrm{in}}}$. The goal is to predict future node fields or labels from the observed sequence.

\paragraph{Overview.}
WaveGSSM first maps each snapshot to a hidden representation $\m Z_t^{(0)}\in\R^{N\times d}$ and then processes the complete hidden sequence with $L$ stacked graph-wave blocks. Block $\ell$ takes $\{\m Z_t^{(\ell-1)}\}_{t=1}^{T}$ as input, performs a recurrent graph-coupled rollout over time, and returns $\{\m Z_t^{(\ell)}\}_{t=1}^{T}$ to the next block. After the final block, a task-specific readout produces the prediction. \Cref{fig_method_frame} (c) illustrates this overall architecture, and \Cref{alg:wavegssm} summarizes the forward pass.

The key difference from a space-then-time architecture (\Cref{fig_method_frame} (b)) is the state evolution within each block. WaveGSSM maintains a \emph{pattern state} $\m U_t^{(\ell)}$, representing the current latent pattern, and a \emph{motion state} $\m V_t^{(\ell)}$, representing its temporal rate of change. Both states are initialized from the incoming hidden sequence $\{\m Z_t^{(\ell-1)}\}_{t=1}^{T}$. At each temporal step, graph interactions and the current hidden input update the motion state, which then advances the pattern state. The evolved pattern is injected back into the hidden stream before being passed to the next block. Thus, $\m Z_{1:T}^{(\ell)}$ propagates information across network depth, while $(\m U_t^{(\ell)},\m V_t^{(\ell)})$ form the recurrent state evolving over time within a block.

We next define the graph propagation operator (\Cref{sec:medium}), the coupled pattern--motion evolution (\Cref{sec:wave}), and its realization as a residual WaveGSSM block with task readout (\Cref{sec:block}).

\subsection{Graph Propagation Operator}
\label{sec:medium}
The propagation operator acts on a support graph $G_{\star}=(\mathcal V,E_{\star})$, which specifies the candidate edges for latent interaction. For fixed-mesh tasks, $G_{\star}$ is the prescribed mesh. For evolving graphs, it is constructed once from the observed structural sequence as detailed in Appendix~\ref{app:support}.

Let $\m B_{\star}\in\R^{|E_{\star}|\times N}$ be an incidence matrix of the support graph, and let $\m a_t\in\{0,1\}^{|E_{\star}|}$ indicate which support edges are active at step $t$.
Block $\ell$ learns the positive base edge strengths $\bar{\m w}^{(\ell)}\in\R^{|E_{\star}|}$ and defines the weighted graph Laplacian
\begin{equation}
\m L_t^{(\ell)}
=
\m B_{\star}^{\top}
\diag
\left(
\m a_t
\odot
\bar{\m w}^{(\ell)}
\right)
\m B_{\star}
\in
\R^{N\times N}.
\label{eq:Lt}
\end{equation}
The resulting operator is sparse, symmetric, and positive semidefinite. It becomes time-dependent when the graph evolves. For a fixed graph, all support edges remain active and the learned operator is shared across time. The operator therefore determines where latent motion can propagate and how strongly neighboring nodes interact.

\subsection{Coupled Pattern--Motion State Evolution}
\label{sec:wave}

Within block $\ell$, we consider continuous pattern and motion states $\m U^{(\ell)}(\tau),\m V^{(\ell)}(\tau)\in\R^{N\times d}$, where $\tau$ denotes continuous latent time. The pattern state represents the current latent pattern, while the motion state represents its rate of change,
\begin{equation}
\m V^{(\ell)}(\tau)
=
\frac{d}{d\tau}
\m U^{(\ell)}(\tau).
\end{equation}
Under the graph-wave formulation, $\m U^{(\ell)}$ and $\m V^{(\ell)}$ are interpreted as displacement and velocity, respectively. For the continuous-time formulation, we denote the graph operator and forcing by $\m L^{(\ell)}(\tau)$ and $\m F^{(\ell)}(\tau)$, whose discrete counterparts are $\m L_t^{(\ell)}$ and $\m F_t^{(\ell)}$.

We motivate their evolution with a forced and damped graph-wave equation
\begin{equation}
\frac{d^2}{d\tau^2}
\m U^{(\ell)}(\tau)
+
\frac{d}{d\tau}
\m U^{(\ell)}(\tau)
\mathbf{\Gamma}^{(\ell)}
+
\m L^{(\ell)}(\tau)
\m U^{(\ell)}(\tau)
\m K^{(\ell)}
=
\m F^{(\ell)}(\tau).
\label{eq:wave_ct}
\end{equation}
Here, $\m K^{(\ell)}\in\R^{d\times d}$ and $\mathbf{\Gamma}^{(\ell)}\in\R^{d\times d}$ are learnable positive diagonal matrices. The former controls graph-interaction strength across latent channels, while the latter controls damping and motion persistence. The forcing $\m F^{(\ell)}(\tau)\in\R^{N\times d}$ introduces information from the input sequence, and the graph force $-\m L^{(\ell)}(\tau)\m U^{(\ell)}(\tau)\m K^{(\ell)}$ couples neighboring nodes.

Using the motion state defined above, Eq.~(\ref{eq:wave_ct}) can equivalently be written as the first-order system
\begin{align}
\frac{d}{d\tau}
\m V^{(\ell)}(\tau)
&=
-
\m V^{(\ell)}(\tau)
\mathbf{\Gamma}^{(\ell)}
-
\m L^{(\ell)}(\tau)
\m U^{(\ell)}(\tau)
\m K^{(\ell)}
+
\m F^{(\ell)}(\tau),
\\
\frac{d}{d\tau}
\m U^{(\ell)}(\tau)
&=
\m V^{(\ell)}(\tau).
\end{align}

To realize these dynamics as a recurrent state transition within each WaveGSSM block, we discretize the system using a semi-implicit Euler scheme~\citep{cromer1981stable}. For block $\ell$, let $\dt^{(\ell)}>0$ be a learnable step size and define $\tau_t^{(\ell)} = t\,\dt^{(\ell)}$ for $t=0,\ldots,T$. We denote the corresponding discrete states by $\m U_t^{(\ell)} \approx \m U^{(\ell)}(\tau_t^{(\ell)})$ and $\m V_t^{(\ell)} \approx \m V^{(\ell)}(\tau_t^{(\ell)})$. At rollout step $t$, $\m L_t^{(\ell)}$ is the graph operator defined in \Cref{sec:medium}, while $\m F_t^{(\ell)}$ is generated from the incoming hidden representation as specified in \Cref{sec:block}. The discrete WaveGSSM transition is
\begin{align}
\m V_t^{(\ell)}
&=
\left[
\m V_{t-1}^{(\ell)}
+
\dt^{(\ell)}
\left(
-
\m L_t^{(\ell)}
\m U_{t-1}^{(\ell)}
\m K^{(\ell)}
+
\m F_t^{(\ell)}
\right)
\right]
\left(
\I_d
+
\dt^{(\ell)}
\mathbf{\Gamma}^{(\ell)}
\right)^{-1},
\label{eq:update_v}
\\
\m U_t^{(\ell)}
&=
\m U_{t-1}^{(\ell)}
+
\dt^{(\ell)}
\m V_t^{(\ell)}.
\label{eq:update_u}
\end{align}
A derivation of Eqs.~(\ref{eq:update_v})--(\ref{eq:update_u}) is given in Appendix~\ref{app:discretization}. The motion state is updated first with damping treated implicitly, and the updated motion then advances the pattern state. Here, $\I_d\in\R^{d\times d}$ is the identity matrix. Because $\mathbf{\Gamma}^{(\ell)}$ is positive and diagonal, the inverse in Eq.~(\ref{eq:update_v}) is implemented by elementwise division. Through this two-stage update, spatial propagation and temporal evolution are coupled within the same transition. Graph interactions update the motion state, and the updated motion advances the pattern, enabling oscillatory and propagating dynamics. Observation~\ref{thm:velocity_phase} illustrates this standard second-order state dependence in the graph-wave setting.

\begin{observation}[Pattern state alone does not determine graph-wave evolution]
\label{thm:velocity_phase}
Let
$\m u(\tau)\in\R^N$
satisfy
\[
\ddot{\m u}(\tau)
+
\gamma\dot{\m u}(\tau)
+
\kappa\m L\m u(\tau)
=
\Zero,
\]
where
$\m L\succeq0$,
$\kappa>0$,
and
$\gamma\geq0$.
For an eigenpair
$(\lambda_i,\m q_i)$
of
$\m L$
with
$\lambda_i>0$,
define
$a_i(\tau)=\m q_i^\top\m u(\tau)$.
If
$\kappa\lambda_i>\gamma^2/4$
and
$\omega_i=(\kappa\lambda_i-\gamma^2/4)^{1/2}$,
then
\begin{equation}
a_i(\tau)
=
e^{-\gamma \tau/2}
\left[
a_i(0)\cos(\omega_i \tau)
+
\frac{
\dot a_i(0)
+
\frac{\gamma}{2}a_i(0)
}{
\omega_i
}
\sin(\omega_i \tau)
\right].
\label{eq:modal_phase_solution}
\end{equation}
The future graph-wave state, therefore, depends on both the displacement $a_i(0)$ and velocity $\dot a_i(0)$, which correspond to modal components of the pattern and motion states in WaveGSSM.
\end{observation}

The proof is given in Appendix~\ref{app:proof_velocity_phase}. Observation~\ref{thm:velocity_phase} explains the phenomenon illustrated in \Cref{fig_intro_rollout}. Two trajectories can share the same pattern state while evolving differently when their motion states differ. WaveGSSM represents both components explicitly through $(\m U_t,\m V_t)$.

For a fixed graph operator, the semi-implicit rollout is non-amplifying for oscillatory graph modes under suitable conditions, as shown in Appendix~\ref{app:stability}.

\subsection{From Graph-Wave Dynamics to WaveGSSM}
\label{sec:block}

Next, we embed the graph-wave transition in a residual neural block. Each observed snapshot is first mapped to hidden width $d$ by a shared node-wise input projection $\phi_{\mathrm{in}}:\R^{N\times F_{\mathrm{in}}}\rightarrow\R^{N\times d}$,
\begin{equation}
\m Z_t^{(0)}
=
\phi_{\mathrm{in}}(\m X_t)
\in
\R^{N\times d},
\qquad
t=1,\ldots,T.
\end{equation}
Block $\ell$ receives the hidden sequence $\{\m Z_t^{(\ell-1)}\}_{t=1}^{T}$ and yields an updated sequence of the same shape.

The first hidden frame initializes the pattern state. When at least two frames are available, their hidden difference initializes the motion state
\begin{align}
\m U_0^{(\ell)}
&=
\psi_U^{(\ell)}
\left(
\LN
\left(
\m Z_1^{(\ell-1)}
\right)
\right),
\\
\m V_0^{(\ell)}
&=
\psi_V^{(\ell)}
\left(
\LN
\left(
\m Z_2^{(\ell-1)}
\right)
-
\LN
\left(
\m Z_1^{(\ell-1)}
\right)
\right).
\end{align}
Here, $\LN$ denotes the layer normalization, and $\psi_U^{(\ell)},\psi_V^{(\ell)}:\R^{N\times d}\rightarrow\R^{N\times d}$ are lightweight node-wise maps with independent parameters. For a single-frame input, we set $\m V_0^{(\ell)}=\Zero$. The initialization, therefore, provides each block with both a latent pattern and an input-dependent motion cue.

Each hidden snapshot also produces a forcing term
\begin{equation}
\m F_t^{(\ell)}
=
g^{(\ell)}
\left(
\m Z_t^{(\ell-1)}
\right)
\in
\R^{N\times d},
\label{eq:forcing_block}
\end{equation}
where $g^{(\ell)}:\R^{N\times d}\rightarrow\R^{N\times d}$ is a lightweight node-wise forcing map. Starting from $(\m U_0^{(\ell)},\m V_0^{(\ell)})$, the block applies Eqs.~(\ref{eq:update_v}) and~(\ref{eq:update_u}) along the input sequence using $\{\m L_t^{(\ell)}\}_{t=1}^{T}$ and $\{\m F_t^{(\ell)}\}_{t=1}^{T}$.

The evolved pattern state is returned to the hidden stream through a residual update
\begin{align}
\hat{\m Z}_t^{(\ell)}
&=
\m Z_t^{(\ell-1)}
+
\Drop
\left(
\LN
\left(
\m U_t^{(\ell)}
\right)
\right),
\label{eq:block_residual_1}
\\
\m Z_t^{(\ell)}
&=
\hat{\m Z}_t^{(\ell)}
+
\Drop
\left(
\FFN
\left(
\LN
\left(
\hat{\m Z}_t^{(\ell)}
\right)
\right)
\right).
\label{eq:block_residual_2}
\end{align}
Here, $\Drop$ denotes dropout, and $\FFN$ is a standard node-wise feed-forward network. The motion state remains part of the internal recurrent state and is passed between consecutive graph-wave transitions. Stacking $L$ blocks alternates motion-carrying state evolution with feature refinement. Full definitions of the projection, initialization, forcing, and residual maps are given in Appendix~\ref{app:method_details}.

\paragraph{Task readout.}
\label{sec:readout}
After $L$ WaveGSSM blocks, the final block produces the  hidden representation $\m Z_T^{(L)}$ and the terminal pattern--motion state $(\m U_T^{(L)},\m V_T^{(L)})$. For the state readout, we form
\begin{equation}
\m H
=
\left[
\m U_T^{(L)},
\m V_T^{(L)}
\right]
\in
\R^{N\times2d}.
\label{eq:readout}
\end{equation}
A task-specific decoder maps the selected representation to node labels or future fields. Temporal node classification uses either $\m H$ or $\m Z_T^{(L)}$, as detailed in Appendix~\ref{app:readout_details}. For mesh-based forecasting, $\m H$ is projected to the decoder width and mapped back to the target grid, conditioning the forecast on both terminal pattern and motion states.

\section{Experiments}
\label{sec:experiments}
\subsection{Experimental Setup}
\label{sec:exp_setup}

We evaluate WaveGSSM on temporal node classification and global weather forecasting. The former tests whether the proposed state evolution generalizes beyond explicitly physical processes, while the latter evaluates the motivating setting of multi-step spatial propagation.

\paragraph{Temporal graph node classification.}
We evaluate WaveGSSM on DBLP-3, Brain, Reddit, and DBLP-10 from the GraphSSM benchmark \citep{li2024graphssm}, reporting Micro-F1 and Macro-F1. Baselines include graph embedding methods, discrete-time temporal GNNs, and GraphSSM. Full dataset, split, model, and implementation details are provided in Appendix~\ref{app:nodeclf}.

\paragraph{Global weather forecasting.}
We construct a daily global forecasting task from ERA5 \citep{hersbach2020era5,lam2023graphcast} on a $1^{\circ}$ latitude--longitude grid with a global icosahedral mesh. The model jointly forecasts geopotential at six pressure levels, daily-maximum 2\,m temperature, and level-1 soil moisture. We train on 1979--2023, validate on 2024, and evaluate on 354 initializations from 2025 using latitude-weighted RMSE and centred anomaly correlation.

Because the original GraphCast setting differs in forecast interval, resolution, and evaluation period, we train GraphCast$^{*}$ under the same daily ERA5 setting as WaveGSSM. We also include GraphSSM \citep{li2024graphssm} and GNN-LSTM as state-space and recurrent baselines. Full experimental details are provided in Appendix~\ref{app:setup}, with additional motion-state analysis in Appendix~\ref{app:weather_velocity_ablation}.

\subsection{Generalization to Temporal Graph Benchmarks}
\label{sec:temporal_graph_results}

We first examine whether the motion-carrying state is also useful for general temporal graph tasks. As shown in \Cref{tab:small_scale_temporal_graphs}, WaveGSSM obtains the highest mean result on all eight reported metrics across DBLP-3, Brain, Reddit, and DBLP-10. The gains are most pronounced on DBLP-3 and Brain, while the margins are smaller on Reddit and DBLP-10. These results show that the proposed state evolution also provides an effective representation mechanism for discrete-time temporal graphs.

\begin{table}[h]
\centering
\caption{Node-classification performance (\%) on four temporal graph benchmarks.
Mean $\pm$ standard deviation over five runs where available.
Higher is better. Best and second-best results are highlighted in
\textcolor{besttext}{\textbf{green}} and \textcolor{secondtext}{\textbf{blue}}, respectively.}
\label{tab:small_scale_temporal_graphs}

\resizebox{\linewidth}{!}{%
\begin{tabular}{@{}lcccccccc@{}}
\toprule
& \multicolumn{2}{c}{DBLP-3}
& \multicolumn{2}{c}{Brain}
& \multicolumn{2}{c}{Reddit}
& \multicolumn{2}{c}{DBLP-10} \\
\cmidrule(lr){2-3}
\cmidrule(lr){4-5}
\cmidrule(lr){6-7}
\cmidrule(lr){8-9}
& Micro-F1 & Macro-F1
& Micro-F1 & Macro-F1
& Micro-F1 & Macro-F1
& Micro-F1 & Macro-F1 \\
\midrule

DeepWalk
& 47.53$_{\pm 0.4}$ & 47.21$_{\pm 0.2}$
& 51.45$_{\pm 0.6}$ & 51.03$_{\pm 0.8}$
& 26.82$_{\pm 0.6}$ & 26.75$_{\pm 0.4}$
& 66.38 & 67.12 \\

Node2Vec
& 48.79$_{\pm 0.3}$ & 48.42$_{\pm 0.4}$
& 53.51$_{\pm 0.5}$ & 52.95$_{\pm 0.6}$
& 25.47$_{\pm 0.6}$ & 25.44$_{\pm 0.5}$
& 67.31 & 66.93 \\

\midrule

HTNE
& 48.98$_{\pm 0.2}$ & 48.74$_{\pm 0.3}$
& 22.31$_{\pm 0.8}$ & 22.12$_{\pm 0.5}$
& 26.96$_{\pm 0.5}$ & 26.80$_{\pm 0.7}$
& 68.79 & 68.36 \\

M$^2$DNE
& 49.12$_{\pm 0.5}$ & 48.87$_{\pm 0.4}$
& 23.79$_{\pm 0.5}$ & 23.54$_{\pm 0.4}$
& 25.79$_{\pm 0.6}$ & 25.61$_{\pm 0.4}$
& 69.71 & 69.75 \\

DynamicTriad
& 48.78$_{\pm 0.5}$ & 48.63$_{\pm 0.6}$
& 21.71$_{\pm 0.7}$ & 21.94$_{\pm 0.7}$
& 28.76$_{\pm 0.5}$ & 28.51$_{\pm 0.5}$
& 66.95 & 66.42 \\

\midrule

MPNN
& 81.78$_{\pm 0.6}$ & 81.46$_{\pm 1.2}$
& 90.97$_{\pm 1.4}$ & 91.01$_{\pm 1.5}$
& 40.85$_{\pm 1.3}$ & 40.64$_{\pm 1.2}$
& 67.74$_{\pm 0.3}$ & 65.05$_{\pm 0.5}$ \\

STAR
& \cellcolor{secondbg}84.74$_{\pm 1.0}$
& \cellcolor{secondbg}84.20$_{\pm 1.2}$
& 92.08$_{\pm 1.3}$ & 92.23$_{\pm 1.3}$
& 43.42$_{\pm 2.3}$ & 43.43$_{\pm 2.4}$
& 72.98$_{\pm 1.5}$ & 72.03$_{\pm 1.2}$ \\

tNodeEmbed
& 84.51$_{\pm 1.2}$ & 83.57$_{\pm 1.1}$
& 92.35$_{\pm 0.8}$ & 92.30$_{\pm 1.0}$
& 42.11$_{\pm 1.8}$ & 42.06$_{\pm 1.3}$
& 74.19$_{\pm 1.8}$ & 74.23$_{\pm 2.2}$ \\

EvolveGCN
& 84.01$_{\pm 1.5}$ & 83.12$_{\pm 1.5}$
& 92.20$_{\pm 1.3}$ & 92.00$_{\pm 1.0}$
& 41.24$_{\pm 1.3}$ & 41.11$_{\pm 1.5}$
& 71.32$_{\pm 0.5}$ & 71.20$_{\pm 0.7}$ \\

SpikeNet
& 83.92$_{\pm 1.5}$ & 83.04$_{\pm 1.1}$
& 92.00$_{\pm 1.2}$ & 91.97$_{\pm 1.2}$
& 40.42$_{\pm 2.0}$ & 40.20$_{\pm 2.1}$
& 74.86$_{\pm 0.5}$ & 74.65$_{\pm 0.5}$ \\

ROLAND
& 84.21$_{\pm 1.4}$ & 84.06$_{\pm 1.5}$
& 92.14$_{\pm 1.2}$ & 91.85$_{\pm 1.1}$
& 44.22$_{\pm 2.2}$ & 44.25$_{\pm 1.9}$
& 75.01$_{\pm 1.1}$ & 74.98$_{\pm 1.0}$ \\

GraphSSM
& 83.52$_{\pm 1.6}$ & 82.98$_{\pm 1.5}$
& \cellcolor{secondbg}92.36$_{\pm 0.8}$
& \cellcolor{secondbg}92.75$_{\pm 0.8}$
& \cellcolor{secondbg}46.79$_{\pm 1.9}$
& \cellcolor{secondbg}46.65$_{\pm 2.1}$
& \cellcolor{secondbg}76.29$_{\pm 0.7}$
& \cellcolor{secondbg}75.51$_{\pm 0.8}$ \\

\midrule

\textbf{WaveGSSM (Ours)}
& \cellcolor{bestbg}$\mathbf{86.34}_{\pm 1.5}$
& \cellcolor{bestbg}$\mathbf{86.26}_{\pm 1.6}$
& \cellcolor{bestbg}$\mathbf{93.36}_{\pm 0.7}$
& \cellcolor{bestbg}$\mathbf{93.64}_{\pm 0.5}$
& \cellcolor{bestbg}$\mathbf{47.47}_{\pm 1.7}$
& \cellcolor{bestbg}$\mathbf{47.51}_{\pm 1.7}$
& \cellcolor{bestbg}$\mathbf{76.38}_{\pm 0.4}$
& \cellcolor{bestbg}$\mathbf{76.12}_{\pm 0.5}$ \\

\bottomrule
\end{tabular}%
}
\end{table}

\subsection{Global Weather Forecasting}
\label{sec:weather_results}
We next evaluate WaveGSSM on global weather forecasting, a real-world setting with propagating atmospheric patterns. The model forecasts eight ERA5 target fields. \Cref{fig:exp01_skill} highlights four variables: the geopotential height at 500 and 200\,hPa, the daily-maximum 2\,m temperature, and the level-1 soil moisture. RMSE measures pointwise error, while anomaly correlation (AC) measures the large-scale pattern agreement. Our results for the remaining pressure levels are given in Appendix~\ref{app:more_forecasting}.

WaveGSSM achieves the lowest RMSE and highest AC for every displayed variable and lead time. Relative to the backbone-matched GraphCast$^{*}$ baseline, it reduces geopotential RMSE by $20.2\%$ on average over all six pressure levels and forecast days~1--5. The lead-time dependence is especially revealing. For z500, the RMSE reduction decreases from $38.1\%$ at day~1 to $7.0\%$ at day~5. In contrast, the AC advantage becomes more pronounced toward the medium range, reaching its strongest level around one week. Thus, while the pointwise-error margin narrows with lead, WaveGSSM better maintains large-scale pattern agreement over medium-range forecasts. For propagating patterns, this distinction matters because a forecast can remain locally plausible even as spatial displacement accumulates over time.

Importantly, WaveGSSM also outperforms GraphSSM and GNN-LSTM across the displayed forecasts.
The advantage, therefore, extends beyond the backbone-matched snapshot comparison and remains present against both graph state-space and recurrent temporal models.
The controlled motion-state ablations in \Cref{sec:velocity_ablation} next examine whether this advantage depends on carrying the proposed motion state.

\begin{figure}[t]
\centering
\includegraphics[width=\linewidth]{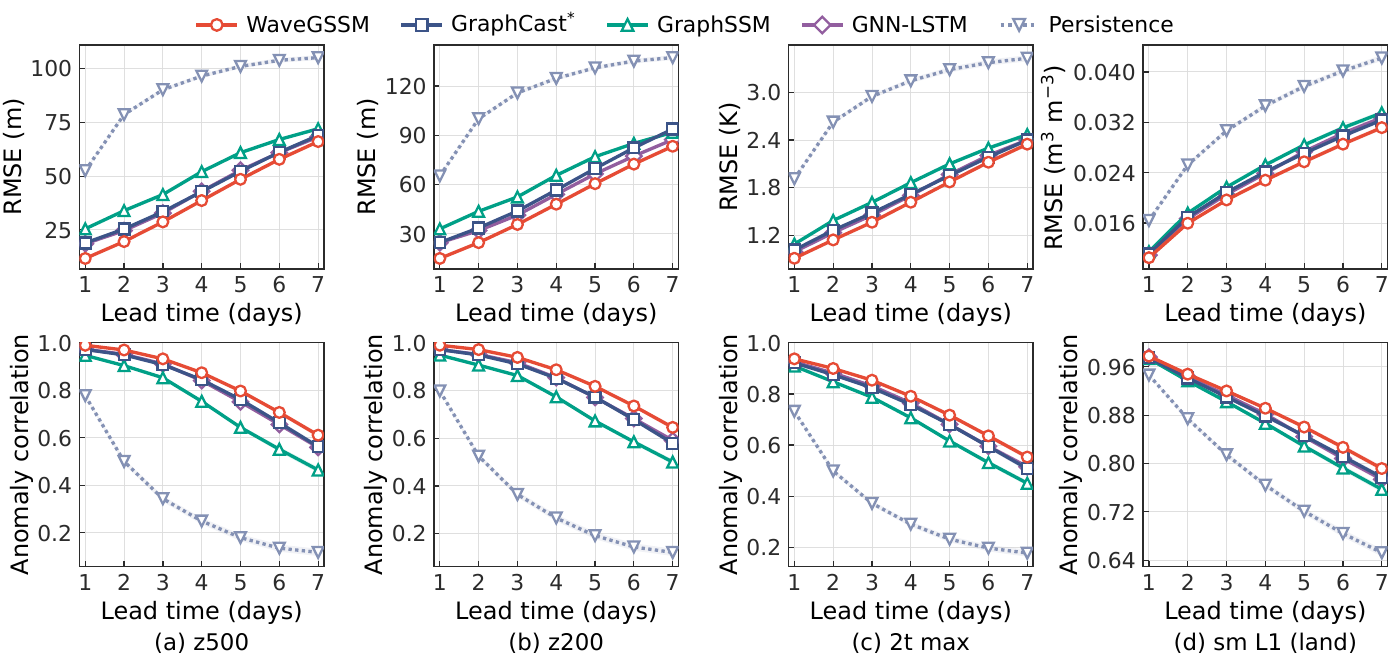}
\caption{Forecast skill on the 2025 ERA5 test set. Latitude-weighted RMSE and centred anomaly correlation over forecast days~1--7 for (a) 500\,hPa geopotential height, (b) 200\,hPa geopotential height, (c) daily-maximum 2\,m temperature, and (d) level-1 soil moisture over land. WaveGSSM achieves the lowest RMSE and highest anomaly correlation across all displayed fields and leads.}
\label{fig:exp01_skill}
\end{figure}

\subsection{Ablating Motion-State Carry}
\label{sec:velocity_ablation}

The previous comparisons involve different model families. Here, we retrain WaveGSSM variants to test whether carrying the motion state contributes to predictive performance.

For a temporal node classification, the no-carry variant preserves the input-derived initial motion state and the instantaneous graph response, but discards the motion state produced by each transition before the next. This removes persistent motion-state carry, while keeping the rest of the model and training procedure unchanged. Implementation details are provided in Appendix~\ref{app:nodeclf_velocity_ablation}.

\begin{table}[t]
\centering
\caption{Retrained ablation of persistent motion-state carry. Both variants use the same dataset-specific architecture, parameter count, training protocol, and validation rule. Values are mean $\pm$ standard deviation over five seeds. Higher is better.}
\label{tab:velocity_carry_ablation}
\setlength{\tabcolsep}{4.0pt}
\renewcommand{\arraystretch}{1.08}
\small
\resizebox{\linewidth}{!}{%
\begin{tabular}{@{}lcccccc@{}}
\toprule
&
\multicolumn{2}{c}{DBLP-3}
&
\multicolumn{2}{c}{Brain}
&
\multicolumn{2}{c}{Reddit} \\
\cmidrule(lr){2-3}
\cmidrule(lr){4-5}
\cmidrule(lr){6-7}
Model
& Micro-F1
& Macro-F1
& Micro-F1
& Macro-F1
& Micro-F1
& Macro-F1 \\
\midrule

\textbf{WaveGSSM}
& \cellcolor{bestbg}$\mathbf{86.34}_{\pm 1.54}$
& \cellcolor{bestbg}$\mathbf{86.26}_{\pm 1.58}$
& \cellcolor{bestbg}$\mathbf{93.36}_{\pm 0.66}$
& \cellcolor{bestbg}$\mathbf{93.64}_{\pm 0.47}$
& \cellcolor{bestbg}$\mathbf{47.47}_{\pm 1.68}$
& \cellcolor{bestbg}$\mathbf{47.51}_{\pm 1.73}$ \\

\begin{tabular}[c]{@{}l@{}}
WaveGSSM w/o 
motion-state carry
\end{tabular}
& $84.51_{\pm 1.41}$
& $84.45_{\pm 1.55}$
& $90.00_{\pm 0.66}$
& $90.65_{\pm 0.92}$
& $45.40_{\pm 0.59}$
& $45.27_{\pm 0.44}$ \\

\bottomrule
\end{tabular}%
}
\end{table}

On the three temporal-graph datasets included in this ablation, removing persistent motion-state carry lowers both Micro-F1 and Macro-F1 by $1.81\%$--$3.36\%$. The ablated model remains competitive on DBLP-3, but falls below GraphSSM on Brain and Reddit. Since it is retrained independently, the remaining gap indicates that persistent motion-state carry contributes to the performance gains.

The same trend holds for weather forecasting. The parameter-matched variant resets the incoming motion state before every internal transition, including the first, while retaining the instantaneous response within each transition. After independent retraining, mean geopotential RMSE over all six pressure levels and forecast days~1--5 increases from $30.72$ to $31.53$\,m, a $2.63\%$ increase. The ablated model also has a higher z200 RMSE at every reported lead from day~1 to day~10. The complete results are given in Appendix~\ref{app:weather_retrained_no_carry}.

\subsection{When Does the Motion-Carrying State Help?}
\label{sec:analysis}

The ablations above show that carrying the motion state substantially contributes to the predictive performance. We now ask when the weather-forecasting advantage becomes strongest and whether it arises from a better propagation of coherent structure rather than stronger spatial smoothing. If motion matters, the gain should increase, when the large-scale wave activity is stronger and as the displacement errors accumulate with the forecast lead.

\begin{figure}[t]
\centering
\includegraphics[width=\linewidth]{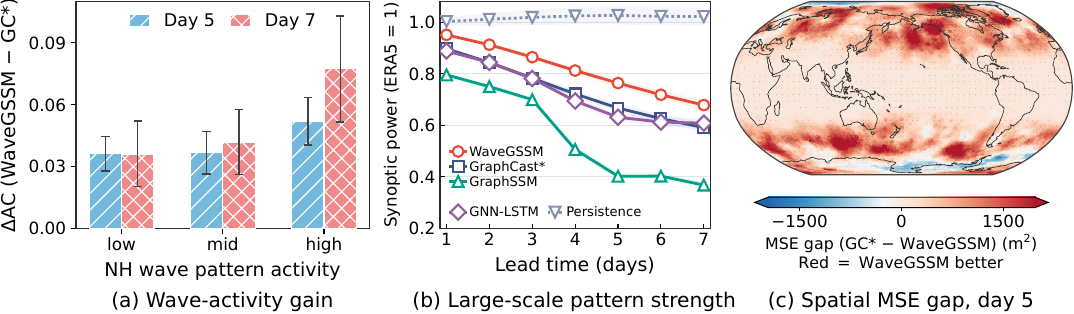}
\caption{When and where WaveGSSM improves, and how well synoptic structure is preserved. Panel (a) shows the z500 anomaly-correlation gain $\Delta\mathrm{AC}=\mathrm{AC}_{\mathrm{WaveGSSM}}-\mathrm{AC}_{\mathrm{GC}^{*}}$ relative to GraphCast$^{*}$. Panel (a) groups days~5 and~7 by Northern-Hemisphere extratropical wave activity. Panel (b) reports z500 forecast power relative to ERA5 over zonal wavenumbers~4--9. Panel (c) shows the day-5 spatial MSE gap, where red favors WaveGSSM. The advantage grows with wave activity, while WaveGSSM preserves more power over the synoptic band.}
\label{fig:analysis}
\end{figure}

\paragraph{The gain grows when propagation matters.}
We stratify the z500 anomaly-correlation advantage by wave activity (\Cref{fig:analysis}a). The advantage increases from low- to high-activity cases and is larger at day~7 than at day~5 in the most active regime. At day~7, using ($1-\mathrm{AC}$) as pattern error, WaveGSSM removes $18\%$ of the error remaining in GraphCast$^{*}$ for the most wave-active third of forecasts, compared with $9\%$ for the least active third. The day-5 spatial error map further shows that the improvement is geographically broad rather than confined to a small number of regions (\Cref{fig:analysis}c). The supporting seasonal and regional breakdowns show larger gains during boreal winter and spring and across the extratropics (Appendix~\Cref{fig:analysis_appendix}). These results emphasize that the advantage becomes stronger when large-scale spatial propagation plays a larger role.

\paragraph{Large-scale spatial variability is better preserved.}
A deterministic forecast can reduce the squared error by smoothing spatial variability toward the mean. Hence, we examine the power of z500 over zonal wavenumbers~4--9, which capture large-scale spatial patterns associated with evolving weather systems (\Cref{fig:analysis}b). At day~7, WaveGSSM retains $68\%$ of the ERA5 variability in this band, compared with $59\%$ for GraphCast$^{*}$ and $37\%$ for GraphSSM. Its day-7 spectrum also remains closest to ERA5 across the selected band (Appendix~\Cref{fig:analysis_appendix}c). Persistence retains much of the initial spatial variability, but forecasts poorly as the atmosphere evolves. Preserving variability alone is therefore insufficient. The corresponding patterns must also evolve toward the correct future states.

These results show that WaveGSSM's advantage grows in regimes where spatial propagation matters and is not explained by stronger spatial smoothing.
The retrained ablations in \Cref{sec:velocity_ablation} test the role of motion-state carry directly, while Appendix~\ref{app:weather_velocity_ablation} examines how the trained weather model uses its internal motion-state pathway.
Additional spatial, vertical, bias, and case-level analyses are provided in Appendices~\ref{app:additional_era5_diagnostics} and~\ref{app:additional_forecast_analysis}.

\section{Conclusion}
\label{sec:conclusion}

We studied spatio-temporal graph prediction from the perspective of state representation, showing that the present field alone can be insufficient when future evolution depends on recent motion. The new WaveGSSM addresses this limitation with coupled pattern and motion states that evolve through a second-order graph-wave transition, allowing spatial propagation and temporal evolution to occur within the same state update. Across temporal-graph benchmarks and daily global weather forecasting, WaveGSSM consistently improves the predictive performance, with particularly strong gains in maintaining large-scale pattern agreement over longer rollouts. These results support explicit motion representation as a useful state-space principle for learning propagating dynamics on graphs.

\subsection*{AI use statement}

Generative AI tools were used to assist with language editing of the manuscript, identifying potentially relevant related work, checking theoretical analyses and mathematical derivations, providing suggestions on figure presentation such as color choices, and assisting the first author with coding. All figures were produced from our own data, and AI tools were not used to generate scientific data or figures. Relevant literature identified with AI assistance was independently verified, and all AI-assisted theoretical content, derivations, and code were checked by the authors. The authors take full responsibility for the final content of this work.

\subsection*{Reproducibility statement}

The appendix provides the information needed to reproduce our experiments. Model architecture and implementation details are given in Appendix~\ref{app:method_details}, while dataset preprocessing, training, and evaluation protocols are provided in Appendices~\ref{app:nodeclf} and~\ref{app:setup}. Ablation procedures are detailed in Appendices~\ref{app:nodeclf_velocity_ablation}, \ref{app:weather_retrained_no_carry}, and~\ref{app:weather_velocity_ablation}. Theoretical derivations and stability analysis are provided in Appendix~\ref{app:theory}.

\bibliography{iclr2027_conference}
\bibliographystyle{iclr2027_conference}

\clearpage
\appendix

\section{Theoretical Analysis}
\label{app:theory}

This section provides the theoretical details supporting WaveGSSM.
We first prove that the present graph field alone does not determine the future of an oscillatory graph mode when its recent motion is unknown.
We then derive the semi-implicit transition used by WaveGSSM and analyze its modal stability.

\subsection{Why the Pattern State Alone Is Not a Complete State}
\label{app:proof_velocity_phase}

\begin{proof}[Proof of Observation~\ref{thm:velocity_phase}]
Since $\m L$ is symmetric, let
$(\lambda_i,\m q_i)$
be an eigenpair satisfying
\[
\m L\m q_i=\lambda_i\m q_i.
\]

Define the projected modal amplitude
\[
a_i(\tau)=\m q_i^\top\m u(\tau).
\]
Projecting the graph-wave equation onto this mode gives
\begin{align}
0
&=
\m q_i^\top
\left(
\ddot{\m u}(\tau)
+
\gamma\dot{\m u}(\tau)
+
\kappa\m L\m u(\tau)
\right)
\\
&=
\m q_i^\top\ddot{\m u}(\tau)
+
\gamma\m q_i^\top\dot{\m u}(\tau)
+
\kappa\m q_i^\top\m L\m u(\tau)
\\
&=
\ddot a_i(\tau)
+
\gamma\dot a_i(\tau)
+
\kappa(\m L\m q_i)^\top\m u(\tau)
\\
&=
\ddot a_i(\tau)
+
\gamma\dot a_i(\tau)
+
\kappa\lambda_i\m q_i^\top\m u(\tau)
\\
&=
\ddot a_i(\tau)
+
\gamma\dot a_i(\tau)
+
\kappa\lambda_i a_i(\tau).
\label{eq:proof_modal_ode}
\end{align}
Thus, each graph mode follows a damped oscillator. Set
\[
a_i(\tau)=e^{-\gamma \tau/2}b_i(\tau).
\]
Then
\begin{align}
\dot a_i(\tau)
&=
e^{-\gamma \tau/2}
\left(
\dot b_i(\tau)-\frac{\gamma}{2}b_i(\tau)
\right),
\\
\ddot a_i(\tau)
&=
e^{-\gamma \tau/2}
\left(
\ddot b_i(\tau)
-
\gamma\dot b_i(\tau)
+
\frac{\gamma^2}{4}b_i(\tau)
\right).
\end{align}
Substituting these expressions into Eq.~(\ref{eq:proof_modal_ode}) yields
\begin{align}
0
&=
\ddot a_i(\tau)
+
\gamma\dot a_i(\tau)
+
\kappa\lambda_i a_i(\tau)
\\
&=
e^{-\gamma \tau/2}
\left(
\ddot b_i(\tau)
-
\gamma\dot b_i(\tau)
+
\frac{\gamma^2}{4}b_i(\tau)
\right)
\\
&\quad
+
\gamma e^{-\gamma \tau/2}
\left(
\dot b_i(\tau)
-
\frac{\gamma}{2}b_i(\tau)
\right)
+
\kappa\lambda_i e^{-\gamma \tau/2}b_i(\tau)
\\
&=
e^{-\gamma \tau/2}
\left[
\ddot b_i(\tau)
-
\gamma\dot b_i(\tau)
+
\frac{\gamma^2}{4}b_i(\tau)
+
\gamma\dot b_i(\tau)
-
\frac{\gamma^2}{2}b_i(\tau)
+
\kappa\lambda_i b_i(\tau)
\right]
\\
&=
e^{-\gamma \tau/2}
\left[
\ddot b_i(\tau)
+
\left(
\kappa\lambda_i-\frac{\gamma^2}{4}
\right)b_i(\tau)
\right]
\\
&=
e^{-\gamma \tau/2}
\left[
\ddot b_i(\tau)+\omega_i^2 b_i(\tau)
\right].
\label{eq:proof_b_ode}
\end{align}
Because
$e^{-\gamma \tau/2}>0$,
Eq.~(\ref{eq:proof_b_ode}) implies
\[
\ddot b_i(\tau)+\omega_i^2 b_i(\tau)=0.
\]
Therefore,
\begin{align}
b_i(\tau)
&=
c_1\cos(\omega_i \tau)
+
c_2\sin(\omega_i \tau),
\\
a_i(\tau)
&=
e^{-\gamma \tau/2}
\left[
c_1\cos(\omega_i \tau)
+
c_2\sin(\omega_i \tau)
\right].
\label{eq:proof_general_solution}
\end{align}
The constants are determined by the initial displacement and velocity:
\begin{align}
a_i(0)
&=
c_1,
\\
\dot a_i(0)
&=
\dot b_i(0)
-
\frac{\gamma}{2}b_i(0)
\\
&=
\omega_i c_2
-
\frac{\gamma}{2}c_1
\\
&=
\omega_i c_2
-
\frac{\gamma}{2}a_i(0),
\\
c_2
&=
\frac{
\dot a_i(0)
+
\frac{\gamma}{2}a_i(0)
}{
\omega_i
}.
\label{eq:proof_c2}
\end{align}
Substituting Eq.~(\ref{eq:proof_c2}) into
Eq.~(\ref{eq:proof_general_solution})
gives Eq.~(\ref{eq:modal_phase_solution}).

It remains to show that equal displacement does not imply equal future evolution. Consider two trajectories satisfying
\[
\m u^{(1)}(0)=\m u^{(2)}(0),
\qquad
\m q_i^\top\dot{\m u}^{(1)}(0)
\ne
\m q_i^\top\dot{\m u}^{(2)}(0).
\]
Using Eq.~(\ref{eq:modal_phase_solution}), their modal difference is
\begin{align}
a_i^{(1)}(\tau)-a_i^{(2)}(\tau)
&=
e^{-\gamma \tau/2}
\left[
\left(
a_i^{(1)}(0)-a_i^{(2)}(0)
\right)
\cos(\omega_i \tau)
\right]
\\
&\quad
+
e^{-\gamma \tau/2}
\left[
\frac{
\dot a_i^{(1)}(0)-\dot a_i^{(2)}(0)
+
\frac{\gamma}{2}
\left(
a_i^{(1)}(0)-a_i^{(2)}(0)
\right)
}{
\omega_i
}
\sin(\omega_i \tau)
\right]
\\
&=
e^{-\gamma \tau/2}
\frac{
\dot a_i^{(1)}(0)-\dot a_i^{(2)}(0)
}{
\omega_i
}
\sin(\omega_i \tau)
\\
&=
e^{-\gamma \tau/2}
\frac{
\m q_i^\top\dot{\m u}^{(1)}(0)
-
\m q_i^\top\dot{\m u}^{(2)}(0)
}{
\omega_i
}
\sin(\omega_i \tau).
\label{eq:proof_phase_gap}
\end{align}
For any time $\tau$ satisfying
$\sin(\omega_i \tau)\ne0$,
the right-hand side is nonzero. Two graph fields can therefore share the same current displacement and still evolve into different future states. A displacement-only representation is not a complete Markov state for the oscillatory graph-wave dynamics considered here.
\end{proof}

Generally, the theorem does not claim that a recurrent model cannot estimate motion from an observed history. It shows that a state intended to determine the future of a second-order propagating process must contain information about both the present field and its recent change. WaveGSSM represents these components explicitly through the pattern and motion states
$(\m U_t,\m V_t)$,
as examined empirically in
\Cref{fig_intro_rollout}.

\subsection{Derivation of the Semi-Implicit Rollout}
\label{app:discretization}

We derive
Eqs.~(\ref{eq:update_v})--(\ref{eq:update_u})
from the first-order system in
\Cref{sec:wave}.
At rollout step
$t$,
the graph interaction and forcing are evaluated explicitly using
$\m U_{t-1}^{(\ell)}$,
$\m L_t^{(\ell)}$,
and
$\m F_t^{(\ell)}$,
while the damping term is evaluated implicitly at
$\m V_t^{(\ell)}$.
The semi-implicit Euler update therefore satisfies
\begin{align}
\frac{
\m V_t^{(\ell)}
-
\m V_{t-1}^{(\ell)}
}{
\dt^{(\ell)}
}
&=
-
\m V_t^{(\ell)}
\mathbf{\Gamma}^{(\ell)}
-
\m L_t^{(\ell)}
\m U_{t-1}^{(\ell)}
\m K^{(\ell)}
+
\m F_t^{(\ell)},
\label{eq:app_discrete_v}
\\
\frac{
\m U_t^{(\ell)}
-
\m U_{t-1}^{(\ell)}
}{
\dt^{(\ell)}
}
&=
\m V_t^{(\ell)}.
\label{eq:app_discrete_u}
\end{align}

Rearranging
Eq.~(\ref{eq:app_discrete_v})
gives
\begin{align}
\m V_t^{(\ell)}
\left(
\I_d
+
\dt^{(\ell)}
\mathbf{\Gamma}^{(\ell)}
\right)
&=
\m V_{t-1}^{(\ell)}
+
\dt^{(\ell)}
\left(
-
\m L_t^{(\ell)}
\m U_{t-1}^{(\ell)}
\m K^{(\ell)}
+
\m F_t^{(\ell)}
\right),
\\
\m V_t^{(\ell)}
&=
\left[
\m V_{t-1}^{(\ell)}
+
\dt^{(\ell)}
\left(
-
\m L_t^{(\ell)}
\m U_{t-1}^{(\ell)}
\m K^{(\ell)}
+
\m F_t^{(\ell)}
\right)
\right]
\left(
\I_d
+
\dt^{(\ell)}
\mathbf{\Gamma}^{(\ell)}
\right)^{-1}.
\end{align}
Similarly,
Eq.~(\ref{eq:app_discrete_u})
gives
\begin{equation}
\m U_t^{(\ell)}
=
\m U_{t-1}^{(\ell)}
+
\dt^{(\ell)}
\m V_t^{(\ell)}.
\end{equation}
These are
Eqs.~(\ref{eq:update_v}) and~(\ref{eq:update_u})
in the main text.

\subsection{Modal Stability of the Semi-Implicit Rollout}
\label{app:stability}

We next analyze the numerical behavior of the semi-implicit rollout under a fixed graph operator. The analysis isolates the integrator by removing the forcing term and holding the graph operator constant. Since
$\m K$
and
$\mathbf{\Gamma}$
are positive diagonal matrices, each graph eigenmode and latent channel can be studied independently. We fix one WaveGSSM block and suppress the block index $\ell$ for clarity.

\begin{proposition}[Modal stability of the semi-implicit rollout]
\label{prop:semi_implicit_stability}
Let
$\m L\succeq0$
be fixed,
\[
\m K=\diag(\kappa_1,\ldots,\kappa_d),
\qquad
\kappa_r>0,
\]
and
\[
\mathbf{\Gamma}
=
\diag(\gamma_1,\ldots,\gamma_d),
\qquad
\gamma_r\geq0.
\]
For a graph eigenvalue
$\lambda_i>0$
and latent channel $r$, define
\[
\omega_{ir}^{2}
=
\kappa_r\lambda_i,
\]
and let
$\dt>0$
denote the step size. The unforced semi-implicit rollout has the modal transition
\begin{equation}
\begin{bmatrix}
a_t\\
b_t
\end{bmatrix}
=
\m A_{ir}
\begin{bmatrix}
a_{t-1}\\
b_{t-1}
\end{bmatrix},
\qquad
\m A_{ir}
=
\frac{1}{1+\dt\gamma_r}
\begin{bmatrix}
1+\dt\gamma_r-(\dt)^2\omega_{ir}^{2} & \dt\\
-\dt\omega_{ir}^{2} & 1
\end{bmatrix}.
\label{eq:modal_transition_matrix}
\end{equation}
If
$\gamma_r=0$
and
$0<\dt\omega_{ir}<2$,
then
\[
\rho(\m A_{ir})=1.
\]
If
$\gamma_r>0$
and
\[
\dt^2\omega_{ir}^{2}<4+2\dt\gamma_r,
\]
then
\[
\rho(\m A_{ir})<1.
\]
In contrast, undamped explicit Euler satisfies
\[
\rho(\m A^{\mathrm{EE}}_{ir})
=
\sqrt{1+(\dt)^2\omega_{ir}^{2}}
>
1
\]
for every $\dt>0$.
\end{proposition}

\begin{proof}
Let
$\m q_i$
be a normalized eigenvector satisfying
\[
\m L\m q_i=\lambda_i\m q_i,
\]
and let
$\m u_{r,t}$
and
$\m v_{r,t}$
denote the $r$-th columns of the pattern and motion states. Define
\[
a_t=\m q_i^\top\m u_{r,t},
\qquad
b_t=\m q_i^\top\m v_{r,t}.
\]
Projecting the unforced update in
Eqs.~(\ref{eq:update_v}) and (\ref{eq:update_u})
onto graph mode $i$ and latent channel $r$ gives
\begin{align}
b_t
&=
\frac{
b_{t-1}
-
h\kappa_r\lambda_i a_{t-1}
}{
1+h\gamma_r
}
\\
&=
\frac{
b_{t-1}
-
h\omega_{ir}^{2}a_{t-1}
}{
1+h\gamma_r
},
\label{eq:modal_velocity_update}
\\
a_t
&=
a_{t-1}
+
hb_t
\\
&=
a_{t-1}
+
h
\frac{
b_{t-1}
-
h\omega_{ir}^{2}a_{t-1}
}{
1+h\gamma_r
}
\\
&=
\frac{
\left(
1+h\gamma_r-h^2\omega_{ir}^{2}
\right)a_{t-1}
+
hb_{t-1}
}{
1+h\gamma_r
}.
\label{eq:modal_displacement_update}
\end{align}
Therefore,
\begin{align}
\begin{bmatrix}
a_t\\
b_t
\end{bmatrix}
&=
\frac{1}{1+\dt\gamma_r}
\begin{bmatrix}
1+h\gamma_r-h^2\omega_{ir}^{2} & h\\
-h\omega_{ir}^{2} & 1
\end{bmatrix}
\begin{bmatrix}
a_{t-1}\\
b_{t-1}
\end{bmatrix}
\\
&=
\m A_{ir}
\begin{bmatrix}
a_{t-1}\\
b_{t-1}
\end{bmatrix}.
\end{align}

The determinant and trace of
$\m A_{ir}$
are
\begin{align}
\det(\m A_{ir})
&=
\frac{
1+h\gamma_r-h^2\omega_{ir}^{2}
+
h^2\omega_{ir}^{2}
}{
(1+h\gamma_r)^2
}
\\
&=
\frac{1}{1+\dt\gamma_r},
\label{eq:modal_determinant}
\\
\operatorname{tr}(\m A_{ir})
&=
\frac{
2+h\gamma_r-h^2\omega_{ir}^{2}
}{
1+h\gamma_r
}.
\label{eq:modal_trace}
\end{align}
Let
\[
\delta_{ir}
=
\det(\m A_{ir}),
\qquad
\theta_{ir}
=
\operatorname{tr}(\m A_{ir}).
\]
The characteristic polynomial is
\begin{equation}
p_{ir}(z)
=
z^2-\theta_{ir}z+\delta_{ir}.
\label{eq:modal_characteristic}
\end{equation}

For
$\gamma_r>0$,
the second-order Jury conditions satisfy
\begin{align}
1-\delta_{ir}
&=
1-\frac{1}{1+\dt\gamma_r}
\\
&=
\frac{
h\gamma_r
}{
1+h\gamma_r
}
>0,
\label{eq:jury_first}
\\
1-\theta_{ir}+\delta_{ir}
&=
1
-
\frac{
2+h\gamma_r-h^2\omega_{ir}^{2}
}{
1+h\gamma_r
}
+
\frac{1}{1+\dt\gamma_r}
\\
&=
\frac{
h^2\omega_{ir}^{2}
}{
1+h\gamma_r
}
>0,
\label{eq:jury_second}
\\
1+\theta_{ir}+\delta_{ir}
&=
1
+
\frac{
2+h\gamma_r-h^2\omega_{ir}^{2}
}{
1+h\gamma_r
}
+
\frac{1}{1+\dt\gamma_r}
\\
&=
\frac{
4+2h\gamma_r-h^2\omega_{ir}^{2}
}{
1+h\gamma_r
}.
\label{eq:jury_third}
\end{align}
If
\[
\dt^2\omega_{ir}^{2}<4+2\dt\gamma_r,
\]
all three quantities are positive. Both roots of
Eq.~(\ref{eq:modal_characteristic})
therefore lie strictly inside the unit disk:
\[
\rho(\m A_{ir})<1.
\]

For
$\gamma_r=0$,
Eqs.~(\ref{eq:modal_determinant}) and~(\ref{eq:modal_trace})
reduce to
\begin{align}
\det(\m A_{ir})
&=
1,
\\
\operatorname{tr}(\m A_{ir})
&=
2-h^2\omega_{ir}^{2}.
\end{align}
Under
$0<h\omega_{ir}<2$,
\[
-2
<
2-h^2\omega_{ir}^{2}
<
2.
\]
The eigenvalues are therefore a complex-conjugate pair with product one, which gives
\[
\rho(\m A_{ir})=1.
\]

For comparison, undamped explicit Euler gives
\begin{align}
\begin{bmatrix}
a_t\\
b_t
\end{bmatrix}
&=
\underbrace{
\begin{bmatrix}
1 & h\\
-h\omega_{ir}^{2} & 1
\end{bmatrix}
}_{\m A^{\mathrm{EE}}_{ir}}
\begin{bmatrix}
a_{t-1}\\
b_{t-1}
\end{bmatrix},
\\
\lambda_{\pm}
\left(
\m A^{\mathrm{EE}}_{ir}
\right)
&=
1
\pm
\mathrm{i}h\omega_{ir},
\\
\left|
\lambda_{\pm}
\left(
\m A^{\mathrm{EE}}_{ir}
\right)
\right|
&=
\sqrt{
1+h^2\omega_{ir}^{2}
}
>
1.
\end{align}
Thus, explicit Euler amplifies every nonzero undamped oscillatory mode, whereas the semi-implicit rollout admits a non-amplifying conservative regime and becomes contractive under positive damping and the stated step-size condition.
\end{proof}

For a frozen graph operator, a sufficient condition covering all nonconstant graph modes and latent channels is
\begin{equation}
(\dt)^2
\kappa_{\max}
\lambda_{\max}(\m L)
<
4+2\dt\gamma_{\min}.
\label{eq:frozen_global_stability}
\end{equation}
where
\[
\kappa_{\max}
=
\max_r\kappa_r,
\qquad
\gamma_{\min}
=
\min_r\gamma_r.
\]
If
$\gamma_{\min}>0$,
this condition makes every nonconstant frozen mode contractive. If some channels have zero damping, the corresponding conservative modes are non-amplifying under
\[
(\dt)^2
\kappa_{\max}
\lambda_{\max}(\m L)
<
4.
\]

This result concerns the graph-wave integrator under a fixed graph operator and without forcing, rather than the full nonlinear forecasting network. It shows that the proposed semi-implicit discretization avoids the systematic amplification of explicit Euler and admits a stable oscillatory regime for the graph-wave component.

\section{WaveGSSM Architecture and Implementation Details}
\label{app:method_details}

This section first summarizes the complete WaveGSSM forward pass and then provides the implementation details for the graph operator, input and forcing maps, wave-state rollout, and task-specific readout.

\subsection{Forward Pass Overview}
\label{app:forward_pass}

\Cref{alg:wavegssm}
summarizes how the components introduced in
\Cref{sec:method}
form a complete forward pass.
Each input snapshot is first mapped to a hidden representation.
Within each WaveGSSM block, the first hidden frames initialize the pattern and motion states.
These states are then advanced sequentially over the observed graph snapshots, after which the evolved pattern state updates the hidden representation passed to the next block.

\begin{algorithm}[t]
\caption{WaveGSSM forward pass}
\label{alg:wavegssm}
\begin{algorithmic}[1]
\Require Observed temporal graph sequence
$\{(G_t,\m X_t)\}_{t=1}^{T}$

\State Encode each snapshot
$\m Z_t^{(0)} \gets \phi_{\mathrm{in}}(\m X_t)$

\For{$\ell=1,\ldots,L$}

    \State Initialize pattern state
    $\m U_0^{(\ell)}$
    from
    $\m Z_1^{(\ell-1)}$

    \State Initialize motion state
    $\m V_0^{(\ell)}$
    from the difference between the first hidden frames

    \For{$t=1,\ldots,T$}

        \State Obtain graph operator
        $\m L_t^{(\ell)}$
        and forcing
        $\m F_t^{(\ell)}$

        \State Update motion state
        $\m V_t^{(\ell)}$
        using Eq.~(\ref{eq:update_v})

        \State Update pattern state
        $\m U_t^{(\ell)}$
        using Eq.~(\ref{eq:update_u})

        \State Update hidden representation
        $\m Z_t^{(\ell)}$
        using
        Eqs.~(\ref{eq:block_residual_1})--(\ref{eq:block_residual_2})

    \EndFor
\EndFor

\State Form the terminal representation from
$\m Z_T^{(L)}$
or
$
[\m U_T^{(L)},\m V_T^{(L)}]
$

\State \Return Task-specific prediction

\end{algorithmic}
\end{algorithm}

The inner loop is the motion-carrying graph rollout across the observed sequence, while the outer loop stacks multiple WaveGSSM blocks.
The motion state is carried between consecutive transitions within a block, whereas the updated hidden sequence is passed between blocks. The following subsections specify each operation in detail.

\subsection{Support Graph Construction}
\label{app:support}

The support graph
$G_{\star}=(\mathcal V,E_{\star})$
defines the candidate edges for latent interaction.
For each snapshot, we convert the raw edge set into a simple undirected graph by removing self-loops, ignoring edge direction, and merging duplicated edges.

For an evolving graph, we keep edges that occur in at least
$k_{\mathrm{sup}}$
snapshots
\begin{equation}
E_{\star}
=
\left\{
e\in\bigcup_{t=1}^{T}E_t
\ \middle|\
\sum_{t=1}^{T}
\mathbf{1}[e\in E_t]
\geq k_{\mathrm{sup}}
\right\}.
\label{eq:support}
\end{equation}
Here,
$k_{\mathrm{sup}}$
is a dataset-specific preprocessing threshold.
Setting
$k_{\mathrm{sup}}=1$
gives the union of all observed edge sets, while larger values remove rarely observed edges.

The support graph is constructed once and shared across WaveGSSM blocks.
Under the transductive node-classification protocol, the complete structural sequence is observed during training and inference, while node labels are not used to construct the support graph.

After assigning an arbitrary orientation to each support edge, we construct the incidence matrix
$\m B_{\star}\in\R^{|E_{\star}|\times N}$.
At snapshot $t$,
$\m a_t\in\{0,1\}^{|E_{\star}|}$
indicates which support edges are active.
Block $\ell$ learns positive base edge strengths
$\bar{\m w}^{(\ell)}\in\R^{|E_{\star}|}$
and forms the weighted graph Laplacian
$\m L_t^{(\ell)}$
defined in
Eq.~(\ref{eq:Lt}).
The support topology is shared across blocks, while the learned edge strengths remain block-specific.

For fixed-mesh tasks, no support-frequency threshold is applied.
The prescribed mesh directly defines
$G_{\star}$,
and all mesh edges remain available throughout the rollout.
This is the setting used for ERA5 forecasting.

\subsection{Input and Forcing Maps}
\label{app:init_force}

\paragraph{Input projection.}
Each observed snapshot is projected to hidden width $d$ as
\begin{equation}
\m Z_t^{(0)}
=
\phi_{\mathrm{in}}(\m X_t)
=
\Drop
\left(
\LN
\left(
\GELU
\left(
\m X_t\m W_{\mathrm{in}}
+
\1_N\m b_{\mathrm{in}}^\top
\right)
\right)
\right)
\in
\R^{N\times d},
\end{equation}
where
$\m W_{\mathrm{in}}\in\R^{F_{\mathrm{in}}\times d}$
and
$\m b_{\mathrm{in}}\in\R^d$.
For block $\ell>1$, the input is the hidden sequence
$\{\m Z_t^{(\ell-1)}\}_{t=1}^{T}$
produced by the preceding block.

\paragraph{Gated linear map.}
The forcing and initialization modules use the gated linear map
\begin{equation}
\GL(\m Z)
=
\left(
\m Z\m W^{(v)}
+
\1_N(\m b^{(v)})^\top
\right)
\odot
\sigma
\left(
\m Z\m W^{(g)}
+
\1_N(\m b^{(g)})^\top
\right).
\label{eq:GL}
\end{equation}
For
$\m Z\in\R^{N\times d_{\mathrm{in}}}$,
the weights satisfy
$\m W^{(v)},\m W^{(g)}
\in\R^{d_{\mathrm{in}}\times d_{\mathrm{out}}}$
and the biases satisfy
$\m b^{(v)},\m b^{(g)}
\in\R^{d_{\mathrm{out}}}$.
Each occurrence of
$\GL$
uses independent parameters.

\paragraph{Forcing.}
The node-wise forcing map
$g^{(\ell)}$
in
Eq.~(\ref{eq:forcing_block})
is implemented as
\begin{equation}
\m F_t^{(\ell)}
=
\GL_{\mathrm{force}}^{(\ell)}
\left(
\LN
\left(
\m Z_t^{(\ell-1)}
\right)
\right)
\in
\R^{N\times d}.
\end{equation}

For ERA5 forecasting, the forcing additionally uses the hidden temporal difference
\begin{equation}
\Delta\m Z_t^{(\ell-1)}
=
\begin{cases}
\Zero_{N\times d}, & t=1,\\[1mm]
\LN
\left(
\m Z_t^{(\ell-1)}
\right)
-
\LN
\left(
\m Z_{t-1}^{(\ell-1)}
\right), & t>1,
\end{cases}
\end{equation}
through
\begin{equation}
\m F_t^{(\ell)}
\leftarrow
\m F_t^{(\ell)}
+
\GL_{\mathrm{imp}}^{(\ell)}
\left(
\Delta\m Z_t^{(\ell-1)}
\right).
\end{equation}
Temporal node classification uses the base forcing map.

\subsection{Wave-State Initialization and Rollout}
\label{app:rollout_details}

\paragraph{Initial state.}
Each WaveGSSM block initializes its wave state from the incoming hidden sequence.
The node-wise initialization map
$\psi_U^{(\ell)}$
in the main text is implemented by
$\GL_{\mathrm{init},u}^{(\ell)}$,
giving
\begin{equation}
\m U_0^{(\ell)}
=
\psi_U^{(\ell)}
\left(
\LN
\left(
\m Z_1^{(\ell-1)}
\right)
\right)
=
\GL_{\mathrm{init},u}^{(\ell)}
\left(
\LN
\left(
\m Z_1^{(\ell-1)}
\right)
\right)
\in
\R^{N\times d}.
\end{equation}
When at least two input snapshots are available, the initial motion state is obtained from their hidden difference
\begin{align}
\m Z_{\mathrm{vel}}^{(\ell)}
&=
\LN
\left(
\m Z_2^{(\ell-1)}
\right)
-
\LN
\left(
\m Z_1^{(\ell-1)}
\right),
\\
\m V_0^{(\ell)}
&=
\psi_V^{(\ell)}
\left(
\m Z_{\mathrm{vel}}^{(\ell)}
\right)
=
\GL_{\mathrm{init},v}^{(\ell)}
\left(
\m Z_{\mathrm{vel}}^{(\ell)}
\right)
\in
\R^{N\times d}.
\end{align}
For a single-frame input,
$\m V_0^{(\ell)}=\Zero_{N\times d}$.
The pattern-state and motion-state initialization maps use independent parameters.

\paragraph{Within-block rollout.}
Starting from
$(\m U_0^{(\ell)},\m V_0^{(\ell)})$,
the block processes
$\{\m L_t^{(\ell)}\}_{t=1}^{T}$
and
$\{\m F_t^{(\ell)}\}_{t=1}^{T}$
using the updates in
Eqs.~(\ref{eq:update_v})
and
(\ref{eq:update_u}).
At each transition, the updated motion state advances the pattern state and is passed directly to the following transition.
The motion state therefore acts as a recurrent state within the block.

The resulting pattern-state trajectory is injected into the hidden stream through
Eqs.~(\ref{eq:block_residual_1})
and
(\ref{eq:block_residual_2}).
The motion state remains internal to the block.
Across network depth, the updated hidden sequence
$\{\m Z_t^{(\ell)}\}_{t=1}^{T}$
is passed to the next block, which initializes and evolves its own wave state.

\paragraph{Autoregressive weather rollout.}
For weather forecasting, the explicit pattern and motion states persist only within each graph-wave processor call.
Across forecast days, predicted prognostic fields form the rolling input history for the next call together with known forcing and static features.
Each new call reconstructs its wave state from this history rather than receiving an explicit
$(\m U,\m V)$
pair from the preceding forecast day.
No future ERA5 target is used during evaluation.

\subsection{Task-Specific Readouts}
\label{app:readout_details}

The final WaveGSSM block provides the terminal hidden representation
$\m Z_T^{(L)}$
and the state readout
$\m H$
defined in
Eq.~(\ref{eq:readout}).
The implementation supports either the
\texttt{state}
readout
$\m H$
or the
\texttt{final}
readout
$\m Z_T^{(L)}$.

For temporal node classification, the selected readout for each dataset is listed in
\Cref{tab:nodeclf_hparams}.
The selected representation is passed through a learned projection followed by a linear node classifier.

ERA5 forecasting uses the
\texttt{state}
readout.
A learned projection maps
$\m H$
to the mesh-decoder width before the mesh-to-grid decoder predicts the target fields.
The frozen-checkpoint readout intervention in
Appendix~\ref{app:weather_velocity_ablation}
scales only the terminal motion-state component
$\m V_T^{(L)}$
while leaving
$\m U_T^{(L)}$,
the decoder parameters, and all tensor dimensions unchanged.

\section{Paired Motion Diagnostic}
\label{app:intro_rollout}

The diagnostic in \Cref{fig_intro_rollout} isolates a basic property of propagating graph signals. Two trajectories are generated on the same irregular graph using the same symmetric propagation operator. They have identical fields at the final observed step but arrive there with opposite recent motion. Their future trajectories therefore differ even though the graph and present field are the same. This construction directly tests whether a model can distinguish futures that cannot be determined from the current snapshot alone.

\paragraph{Compared models.}
We compare a current-only GCN, GCN-LSTM, GCN-SSM, WaveGSSM without the motion state, and the full WaveGSSM. The temporal models receive the same three observed fields, while the current-only GCN receives only the shared final field. The WaveGSSM variant without the motion state retains the graph architecture and hidden width of the full model but removes the motion-state input and its recurrent propagation. All models are trained with the same free-running multi-step objective and have closely matched parameter counts. Training, validation, and testing use disjoint graphs, and each result is averaged over five random seeds.

\paragraph{Evaluation.}
The visualized example is selected using only properties of the ground-truth trajectories rather than model performance. Forecast position difference measures the separation between the predicted right-moving and left-moving patterns along the transport path. A value near zero indicates that the model produces nearly the same future for the two histories. Centroid error measures the distance between the predicted and true pattern locations. Panel~(h) in \Cref{fig_intro_rollout} directly compares the rollout errors of WaveGSSM with and without the motion state, while \Cref{fig:intro_rollout_aggregate} reports aggregate results over all held-out pairs and seeds.

\begin{figure}[h]
\centering
\includegraphics[width=\linewidth]{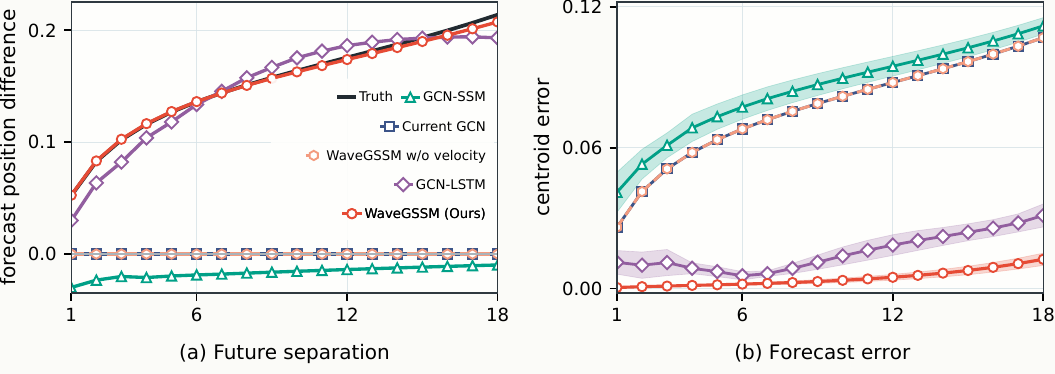}
\caption{Aggregate paired-motion diagnostic on the held-out graph. Panel (a) reports the signed separation between the predicted right-moving and left-moving pattern centres; the black curve is the ground-truth mean over 48 held-out trajectory pairs. Each model curve first averages over the 48 pairs within a training seed and then over five seeds. Panel (b) reports mean centroid error over the corresponding 96 directional trajectories.}
\label{fig:intro_rollout_aggregate}
\end{figure}

\section{Experimental Setup}
\label{app:experiments}

We evaluate WaveGSSM in two complementary settings. Temporal node classification examines whether the proposed state evolution is useful beyond explicitly physical forecasting, while global ERA5 forecasting evaluates the motivating setting in which coherent spatial patterns must be advanced over multiple prediction steps. This section gives the datasets, evaluation protocols, model configurations, and metrics used in both settings.

\subsection{Temporal Node Classification}
\label{app:nodeclf}

\paragraph{Task and datasets.}
We evaluate WaveGSSM on the four discrete-time temporal graph benchmarks adopted by GraphSSM
\citep{li2024graphssm}, including DBLP-3, Brain, Reddit, and DBLP-10.
Each dataset contains a sequence of graph snapshots on a shared node set together with node features and node-level class labels.
The task is transductive node classification after processing the observed graph sequence.

\Cref{tab:nodeclf_data} summarizes the dataset statistics.
The structural temporal-continuity score
$\mathrm{TC}_{\mathrm{s}}$
is the mean Jaccard overlap between consecutive edge sets, while
$\mathrm{TC}_{\mathrm{f}}$
is the mean cosine similarity between consecutive node-feature matrices
\citep{li2024graphssm}.
The datasets span different temporal regimes, from rapidly changing connectivity in Brain to substantially more persistent structure in DBLP-10.

\begin{table}[h]
\centering
\caption{Statistics of the temporal node-classification benchmarks \citep{li2024graphssm}. TC$_{\mathrm{s}}$ and TC$_{\mathrm{f}}$ measure temporal continuity of graph structure and node features. The final column gives the train, validation, and test proportions used for WaveGSSM.}
\label{tab:nodeclf_data}
\small
\begin{tabular}{lrrrrrccl}
\toprule
Dataset & \#Nodes & \#Edges & \#Feat. & \#Cls. & \#Steps & TC$_{\mathrm{s}}$ & TC$_{\mathrm{f}}$ & Split \\
\midrule
DBLP-3  & 4{,}257  & 23{,}540    & 100 & 3  & 10 & 0.139 & 0.468 & 81/9/10 \\
Brain   & 5{,}000  & 1{,}955{,}488 & 20  & 10 & 12 & 0.024 & 0.070 & 81/9/10 \\
Reddit  & 8{,}291  & 264{,}050   & 20  & 4  & 10 & 0.030 & 0.556 & 81/9/10 \\
DBLP-10 & 28{,}085 & 236{,}894   & 128 & 10 & 27 & 0.823 & 0.823 & 75/5/20 \\
\bottomrule
\end{tabular}
\end{table}

\paragraph{Protocol and evaluation.}
We follow the transductive setting of
\citet{li2024graphssm}.
Graph structure and node features from all snapshots are available during training and inference, while supervision is restricted to the training nodes.
The support graph can therefore be constructed from the observed structural sequence without using validation or test labels.

For WaveGSSM, we use the train, validation, and test proportions listed in
\Cref{tab:nodeclf_data}.
Validation nodes are used only for model selection and early stopping.
WaveGSSM is trained with cross-entropy loss on the training nodes.
For each seed, we select the checkpoint with the highest validation Micro-F1 and report the corresponding test Micro-F1 and Macro-F1.
Results are averaged over five seeds from $0$ to $4$, with population standard deviation.

\paragraph{Baselines and result provenance.}
We compare against the baseline families reported by GraphSSM
\citep{li2024graphssm}.
Static graph embedding methods include DeepWalk
\citep{perozzi2014deepwalk}
and Node2Vec
\citep{grover2016node2vec}.
Temporal graph embedding methods include HTNE
\citep{zuo2018embedding},
M$^2$DNE
\citep{lu2019temporal},
and DynamicTriad
\citep{zhou2018dynamic}.
Discrete-time temporal GNN baselines include MPNN
\citep{panagopoulos2021transfer},
STAR \citep{xu2019spatio}, tNodeEmbed \citep{singer2019node}, EvolveGCN \citep{pareja2020evolvegcn}, SpikeNet \citep{li2023scaling}, and ROLAND \citep{you2022roland}. GraphSSM \citep{li2024graphssm} is the closest state-space comparison.

The results for these legacy baselines in \Cref{tab:small_scale_temporal_graphs} are taken from the GraphSSM benchmark. The GraphSSM row reports results from our rerun of the baseline.

\paragraph{WaveGSSM configuration.}
For temporal node classification, the observed graph at snapshot $t$ determines the activity vector $\m a_t$ and the corresponding operator $\m L_t^{(\ell)}$. We set $k_{\mathrm{sup}}=2$ for DBLP-3 and Brain, $k_{\mathrm{sup}}=1$ for Reddit, and $k_{\mathrm{sup}}=4$ for DBLP-10. These thresholds retain edges observed in at least the specified number of snapshots.

Each node snapshot is first encoded by a node-wise MLP and then processed by $L$ stacked WaveGSSM blocks with a learned initial motion state. No auxiliary spatial mixer is used. All four datasets use an FFN expansion ratio of $2$ and the \texttt{state} readout, which concatenates the terminal pattern and motion states, as described in \Cref{app:readout_details}. The selected representation is passed through a learned projection followed by a linear node classifier.

Training uses AdamW, gradient clipping at $1.0$, a \textsc{ReduceLROnPlateau} learning-rate schedule, and validation-based early stopping. Architecture and optimization hyperparameters are selected separately for each dataset using only its validation split. The resulting configurations are listed in \Cref{tab:nodeclf_hparams}.

\begin{table}[h]
\centering
\caption{WaveGSSM configurations for temporal node classification. Here, $d$ denotes hidden width and $L$ denotes the number of WaveGSSM blocks. All runs use AdamW, gradient clipping at $1.0$, validation-based early stopping, and a learned initial motion state.}
\label{tab:nodeclf_hparams}
\small
\begin{tabular}{lccccccc}
\toprule
Dataset & $d$ & $L$ & Epochs & Learning rate & Weight decay & Dropout & Readout \\
\midrule
DBLP-3  & 88  & 3 & 200 & $4.0\times10^{-3}$ & $5\times10^{-4}$ & 0.12 & state \\
Brain   & 224 & 2 & 220 & $4.8\times10^{-3}$ & $8\times10^{-4}$ & 0.20 & state \\
Reddit  & 144 & 6 & 400 & $1.2\times10^{-3}$ & $5\times10^{-5}$ & 0.05 & state \\
DBLP-10 & 280 & 2 & 480 & $4.5\times10^{-3}$ & $5\times10^{-4}$ & 0.13 & state \\
\bottomrule
\end{tabular}
\end{table}

\subsection{ERA5 Global Weather Forecasting}
\label{app:setup}

The ERA5 experiment evaluates the motivating forecasting setting of WaveGSSM, where a multivariate spatial field must be advanced over several autoregressive steps. We evaluate both pointwise forecast accuracy and agreement in large-scale anomaly patterns over a ten-day rollout.

\paragraph{Data and target variables.}
We use daily ERA5 reanalysis
\citep{hersbach2020era5}
regridded to a regular
$1^{\circ}$
latitude--longitude grid with
$360\times181$
grid points.
The prognostic state contains eight target channels.
Six are geopotential at
$150$,
$200$,
$300$,
$500$,
$850$,
and
$1000$\,hPa.
The remaining two are daily-maximum 2\,m temperature and level-1 volumetric soil moisture.

At each forecast step, the model also receives top-of-atmosphere incident solar radiation and sine/cosine encodings of the annual and daily cycles.
Static inputs consist of surface geopotential and the land--sea mask.
Dynamic channels are standardized using statistics from the training period and transformed back to physical units for evaluation.

\paragraph{Forecast protocol.}
We train on
1979-01-01
through
2023-12-31,
use 2024 for validation, and reserve 2025 for testing.
Each forecast starts from two consecutive daily states and advances the prognostic state by one day per autoregressive step.
During rollout, the predicted prognostic state is fed back into the next input together with known forcing and static features.
No future ERA5 target is used after initialization.

Training uses a three-step free-running objective.
Evaluation uses a ten-day autoregressive rollout.
Requiring a complete verification window gives
$354$
test initializations in 2025.
Atmospheric variables are evaluated globally, while soil moisture is evaluated only over land.

\paragraph{Graph and model configuration.}
All weather models use the same
$1^{\circ}$
grid and global level-5 icosahedral mesh following the graph construction of GraphCast
\citep{lam2023graphcast}.
The mesh contains
$10{,}242$
nodes.
The grid-to-mesh encoder contains $101{,}892$ directed edges, the mesh processor contains $81{,}900$ directed edges, and the mesh-to-grid decoder contains $195{,}480$ directed edges. Each graph edge carries four geometric features describing relative spherical displacement and distance, while each mesh node receives three geometric features.

For ERA5, $G_{\star}$ is the prescribed icosahedral mesh and all mesh edges remain available throughout the rollout. WaveGSSM uses hidden width $d=128$ and three processor blocks. Node and edge transformations use single-hidden-layer MLPs. The two input states initialize the pattern and motion states as described in \Cref{app:rollout_details}, and the forcing uses the temporal-difference term defined in \Cref{app:init_force}.

\paragraph{Baselines and controlled comparisons.}
GraphCast$^{*}$ is our central backbone-matched snapshot baseline.
It uses the same grid-to-mesh encoder, mesh-to-grid decoder, hidden width, processor depth, target variables, two-frame history, training horizon, loss, and optimization protocol as WaveGSSM.
It predicts future representations without the graph-wave pattern--motion transition.
The asterisk denotes our daily ERA5 implementation rather than the published GraphCast checkpoint or its original operational configuration.
We therefore refer to it as a GraphCast-style backbone-matched baseline.

GraphSSM \citep{li2024graphssm} is adapted to the same forecasting input and output interface using two graph state-space blocks with state dimension $16$. GNN-LSTM combines per-frame graph processing with a node-shared single-layer LSTM.
Both use the same forecast variables, data split, mesh interface, training targets, and evaluation protocol as WaveGSSM.
Because they retain their native temporal processors, we treat them as cross-model comparisons rather than controlled component ablations.
Persistence and climatology are included as non-learned references.
Persistence keeps the initialized prognostic state fixed, while climatology uses the ERA5 day-of-year mean computed from the 1979--2023 training period.

The retrained motion-state-carry ablation and frozen-checkpoint interventions are detailed in Appendices~\ref{app:weather_retrained_no_carry} and~\ref{app:weather_velocity_ablation}.

\paragraph{Training objective and optimization.}
All matched weather runs use a latitude-weighted multivariate mean-squared-error objective.
Grid-point weights are proportional to
$\cos\varphi$
and normalized to have unit mean.
Each target channel is weighted inversely by the variance of its daily increments over the training period.
Geopotential levels additionally use pressure-dependent weights, while the two surface channels use fixed task weights.
The loss is averaged over target variables, valid grid points, and the three free-running training steps.

We optimize with AdamW using an initial learning rate of $6\times10^{-4}$. Training runs for $300$ epochs on 4 H100 GPUs with a batch size of $4$ per device.
Gradient checkpointing is used inside the WaveGSSM processor.
The full WaveGSSM and retrained motion-state-carry ablation are selected independently on the 2024 validation set using the same checkpoint-selection rule.
All frozen interventions use the selected full WaveGSSM checkpoint without additional training or model selection.

\paragraph{Evaluation metrics.}
For forecast field $\hat y$ and ERA5 target $y$,
the latitude-weighted RMSE is
\begin{equation}
\operatorname{RMSE}
=
\left[
\frac{
\sum_i
w_i
\left(
\hat y_i-y_i
\right)^2
}{
\sum_i w_i
}
\right]^{1/2},
\qquad
w_i=\cos\varphi_i.
\label{eq:weather_rmse}
\end{equation}
The sum is taken over valid grid points.
Ocean points are excluded for soil moisture.
RMSE is computed separately for each target variable and forecast lead.

We compute anomaly correlation relative to the
1979--2023
ERA5 day-of-year climatology.
Let
$\hat a_i$
and
$a_i$
denote forecast and observed anomalies, with latitude-weighted spatial means
\begin{equation}
\overline{\hat a}_w
=
\frac{\sum_i w_i\hat a_i}{\sum_i w_i},
\qquad
\overline{a}_w
=
\frac{\sum_i w_i a_i}{\sum_i w_i}.
\end{equation}
The centred anomaly correlation is
\begin{equation}
\operatorname{AC}
=
\frac{
\sum_i
w_i
\left(
\hat a_i-\overline{\hat a}_w
\right)
\left(
a_i-\overline{a}_w
\right)
}{
\sqrt{
\sum_i
w_i
\left(
\hat a_i-\overline{\hat a}_w
\right)^2
}
\sqrt{
\sum_i
w_i
\left(
a_i-\overline{a}_w
\right)^2
}
}.
\label{eq:weather_ac}
\end{equation}
AC measures agreement in the spatial anomaly pattern after removing the weighted spatial mean.
We use it together with RMSE and do not interpret it as a direct measure of physical phase.

\section{Additional Temporal-Graph Details}
\label{app:additional_temporal}

This section provides the implementation details of the matched no-carry variant used in \Cref{sec:velocity_ablation}.

\subsection{Matched No-Carry Variant}
\label{app:nodeclf_velocity_ablation}

\paragraph{Matched no-carry variant.}
We construct a controlled variant, denoted WaveGSSM w/o motion-state carry. It retains the support graph, learned graph operator, forcing maps, stiffness, damping, initial-motion-state map, hidden width, number of blocks, task readout, classifier, loss, and optimization protocol of the full model. The only change is the state passed from one rollout step to the next.

Let the motion state supplied to transition $t$ be
\begin{equation}
\overline{\m V}_{t-1}^{(\ell)}
=
\begin{cases}
\m V_0^{(\ell)}, & t=1,\\
\Zero_{N\times d}, & t>1.
\end{cases}
\label{eq:no_velocity_input}
\end{equation}
The input-derived initial motion cue is therefore preserved, but a motion state produced during the rollout is not carried into the following transition. The instantaneous response at step $t$ is
\begin{equation}
\widetilde{\m V}_t^{(\ell)}
=
\left[
\overline{\m V}_{t-1}^{(\ell)}
+
\dt^{(\ell)}
\left(
-
\m L_t^{(\ell)}
\m U_{t-1}^{(\ell)}
\m K^{(\ell)}
+
\m F_t^{(\ell)}
\right)
\right]
\left(
\I_d
+
\dt^{(\ell)}
\mathbf{\Gamma}^{(\ell)}
\right)^{-1},
\label{eq:no_velocity_response}
\end{equation}
followed by the pattern-state update
\begin{equation}
\m U_t^{(\ell)}
=
\m U_{t-1}^{(\ell)}
+
\dt^{(\ell)}
\widetilde{\m V}_t^{(\ell)}.
\label{eq:no_velocity_displacement}
\end{equation}

For every transition after initialization,
$\overline{\m V}_{t-1}^{(\ell)}=\Zero$.
Eliminating the instantaneous auxiliary variable then gives
\begin{equation}
\m U_t^{(\ell)}
=
\m U_{t-1}^{(\ell)}
+
\left(
\dt^{(\ell)}
\right)^2
\left(
-
\m L_t^{(\ell)}
\m U_{t-1}^{(\ell)}
\m K^{(\ell)}
+
\m F_t^{(\ell)}
\right)
\left(
\I_d
+
\dt^{(\ell)}
\mathbf{\Gamma}^{(\ell)}
\right)^{-1},
\qquad t>1.
\label{eq:no_velocity_first_order}
\end{equation}
The recurrent state in Eq.~(\ref{eq:no_velocity_first_order}) is first-order. The next pattern state depends on the preceding pattern state and the current graph-conditioned forcing, but not on a motion state carried from the preceding step.

The ablated model preserves the same external block interface as the full model. Its terminal instantaneous response $\widetilde{\m V}_T^{(\ell)}$ occupies the motion-state slot wherever that tensor is required by the implementation, and the task-specific readout is otherwise unchanged. The comparison therefore removes persistent motion-state recurrence without changing the model width, decoder interface, or number of learnable parameters. The corresponding results are reported in \Cref{tab:velocity_carry_ablation}.

\section{Additional Weather Forecasting Results and Ablations}
\label{app:additional_weather}

This section reports the complete ERA5 forecast results and evaluates the role of the motion state. We first provide lead-wise skill for all predicted fields. We then evaluate a separately retrained model without motion-state carry and use frozen-checkpoint interventions to examine how the trained full model uses its internal motion-state pathway and terminal motion-state readout. All experiments use the 2025 ERA5 test set and the evaluation protocol in \Cref{app:setup}.

\subsection{Lead-Wise Forecast Skill}
\label{app:more_forecasting}

\Cref{tab:main_fields} reports the four fields highlighted in the main text, while \Cref{tab:more_levels} gives the remaining geopotential levels. For RMSE, we report the relative reduction of WaveGSSM over GraphCast$^{*}$
\begin{equation}
\operatorname{Reduction}_{\mathrm{RMSE}}
=
\frac{
\operatorname{RMSE}_{\mathrm{GC}^{*}}
-
\operatorname{RMSE}_{\mathrm{WaveGSSM}}
}{
\operatorname{RMSE}_{\mathrm{GC}^{*}}
}
\times 100\%.
\label{eq:relative_rmse_reduction}
\end{equation}
For anomaly correlation, we report the absolute gain
\begin{equation}
\Delta\mathrm{AC}
=
\mathrm{AC}_{\mathrm{WaveGSSM}}
-
\mathrm{AC}_{\mathrm{GC}^{*}}.
\label{eq:absolute_ac_gain}
\end{equation}
The absolute AC difference avoids inflating the apparent improvement as the reference correlation decreases at longer leads.

The two metrics reveal different lead-time trends. RMSE reductions are largest at short range and generally narrow with lead, while the AC advantage grows into the medium range and remains positive through day~10. WaveGSSM therefore retains better agreement with the observed anomaly pattern even as the pointwise RMSE margin becomes smaller.

WaveGSSM has lower RMSE for every reported geopotential and soil-moisture entry. For 2\,m daily maximum temperature, the models are effectively tied at day~10, where the rounded RMSE values differ by only $0.001$\,K. WaveGSSM has higher anomaly correlation for every reported field and lead.

\begin{table}[t]
\centering
\caption{Lead-wise forecast skill for the four fields highlighted in the main text on the 2025 ERA5 test set. The upper block reports latitude-weighted RMSE and its relative reduction under WaveGSSM. The lower block reports centred anomaly correlation and the absolute gain $\Delta\mathrm{AC}=\mathrm{AC}_{\mathrm{WaveGSSM}}-\mathrm{AC}_{\mathrm{GC}^{*}}$. RMSE units are metres for $z500$ and $z200$, kelvin for 2\,m daily maximum temperature, and $10^{-3}\,\mathrm{m}^{3}\mathrm{m}^{-3}$ for soil moisture layer~1.}
\label{tab:main_fields}
\setlength{\tabcolsep}{4.2pt}
\renewcommand{\arraystretch}{1.12}
\footnotesize
\resizebox{\textwidth}{!}{%
\begin{tabular}{@{}llcccccccccc@{}}
\toprule
\multirow{2}{*}{Field}
&
\multirow{2}{*}{Model}
&
\multicolumn{10}{c}{Lead time (days)}
\\
\cmidrule(lr){3-12}
&
&
1 & 2 & 3 & 4 & 5 & 6 & 7 & 8 & 9 & 10
\\
\midrule

\multicolumn{12}{@{}l}{\emph{RMSE}~$\downarrow$}
\\[1pt]

\multirow{3}{*}{$z500$ (m)}
& WaveGSSM
& 11.7 & 19.6 & 28.7 & 38.6 & 48.5 & 57.8 & 66.1 & 73.3 & 79.3 & 84.1
\\
& GraphCast$^{*}$
& 18.9 & 25.4 & 33.4 & 42.8 & 52.2 & 61.1 & 68.9 & 75.5 & 80.7 & 85.0
\\
& \impc{Reduction}
& \impc{38.1\%}
& \impc{22.9\%}
& \impc{14.1\%}
& \impc{9.6\%}
& \impc{7.0\%}
& \impc{5.3\%}
& \impc{4.1\%}
& \impc{2.8\%}
& \impc{1.7\%}
& \impc{1.1\%}
\\

\addlinespace[2pt]

\multirow{3}{*}{$z200$ (m)}
& WaveGSSM
& 15.1 & 24.7 & 35.8 & 48.1 & 60.6 & 72.5 & 83.4 & 93.0 & 101.4 & 108.2
\\
& GraphCast$^{*}$
& 24.7 & 33.6 & 44.0 & 56.7 & 69.8 & 82.4 & 93.9 & 103.6 & 111.6 & 118.3
\\
& \impc{Reduction}
& \impc{39.1\%}
& \impc{26.6\%}
& \impc{18.6\%}
& \impc{15.2\%}
& \impc{13.2\%}
& \impc{12.0\%}
& \impc{11.2\%}
& \impc{10.2\%}
& \impc{9.2\%}
& \impc{8.5\%}
\\

\addlinespace[2pt]

\multirow{3}{*}{$2$t\,max (K)}
& WaveGSSM
& 0.910 & 1.142 & 1.364 & 1.616 & 1.872 & 2.120 & 2.348 & 2.551 & 2.740 & 2.903
\\
& GraphCast$^{*}$
& 1.014 & 1.263 & 1.480 & 1.715 & 1.952 & 2.184 & 2.400 & 2.594 & 2.762 & 2.902
\\
& \impc{Reduction}
& \impc{10.2\%}
& \impc{9.6\%}
& \impc{7.8\%}
& \impc{5.8\%}
& \impc{4.1\%}
& \impc{2.9\%}
& \impc{2.2\%}
& \impc{1.7\%}
& \impc{0.8\%}
& \noimp{0.0\%}
\\

\addlinespace[2pt]

\multirow{3}{*}{sm\,L1 ($10^{-3}$)}
& WaveGSSM
& 10.50 & 15.98 & 19.67 & 22.81 & 25.73 & 28.53 & 31.15 & 33.59 & 35.87 & 37.97
\\
& GraphCast$^{*}$
& 11.12 & 16.94 & 20.86 & 24.13 & 27.08 & 29.83 & 32.36 & 34.72 & 36.97 & 39.05
\\
& \impc{Reduction}
& \impc{5.5\%}
& \impc{5.7\%}
& \impc{5.7\%}
& \impc{5.5\%}
& \impc{5.0\%}
& \impc{4.4\%}
& \impc{3.7\%}
& \impc{3.3\%}
& \impc{3.0\%}
& \impc{2.8\%}
\\

\midrule

\multicolumn{12}{@{}l}{\emph{Centred anomaly correlation}~$\uparrow$}
\\[1pt]

\multirow{3}{*}{$z500$}
& WaveGSSM
& 0.989 & 0.969 & 0.932 & 0.874 & 0.797 & 0.707 & 0.611 & 0.515 & 0.427 & 0.353
\\
& GraphCast$^{*}$
& 0.971 & 0.949 & 0.908 & 0.844 & 0.761 & 0.663 & 0.562 & 0.467 & 0.384 & 0.313
\\
& \impc{$\Delta$AC}
& \impc{+0.018}
& \impc{+0.020}
& \impc{+0.024}
& \impc{+0.030}
& \impc{+0.036}
& \impc{+0.044}
& \impc{+0.049}
& \impc{+0.048}
& \impc{+0.043}
& \impc{+0.040}
\\

\addlinespace[2pt]

\multirow{3}{*}{$z200$}
& WaveGSSM
& 0.989 & 0.971 & 0.938 & 0.887 & 0.817 & 0.735 & 0.645 & 0.554 & 0.469 & 0.395
\\
& GraphCast$^{*}$
& 0.971 & 0.949 & 0.911 & 0.851 & 0.772 & 0.678 & 0.579 & 0.486 & 0.405 & 0.336
\\
& \impc{$\Delta$AC}
& \impc{+0.018}
& \impc{+0.022}
& \impc{+0.027}
& \impc{+0.036}
& \impc{+0.045}
& \impc{+0.057}
& \impc{+0.066}
& \impc{+0.068}
& \impc{+0.064}
& \impc{+0.059}
\\

\addlinespace[2pt]

\multirow{3}{*}{$2$t\,max}
& WaveGSSM
& 0.937 & 0.899 & 0.853 & 0.791 & 0.717 & 0.636 & 0.553 & 0.473 & 0.392 & 0.321
\\
& GraphCast$^{*}$
& 0.921 & 0.875 & 0.824 & 0.758 & 0.680 & 0.594 & 0.508 & 0.424 & 0.350 & 0.287
\\
& \impc{$\Delta$AC}
& \impc{+0.016}
& \impc{+0.024}
& \impc{+0.029}
& \impc{+0.033}
& \impc{+0.037}
& \impc{+0.042}
& \impc{+0.045}
& \impc{+0.049}
& \impc{+0.042}
& \impc{+0.034}
\\

\addlinespace[2pt]

\multirow{3}{*}{sm\,L1}
& WaveGSSM
& 0.978 & 0.948 & 0.920 & 0.891 & 0.860 & 0.827 & 0.792 & 0.756 & 0.720 & 0.685
\\
& GraphCast$^{*}$
& 0.975 & 0.941 & 0.910 & 0.879 & 0.846 & 0.812 & 0.777 & 0.742 & 0.707 & 0.673
\\
& \impc{$\Delta$AC}
& \impc{+0.003}
& \impc{+0.007}
& \impc{+0.010}
& \impc{+0.012}
& \impc{+0.014}
& \impc{+0.015}
& \impc{+0.015}
& \impc{+0.014}
& \impc{+0.013}
& \impc{+0.012}
\\

\bottomrule
\end{tabular}%
}
\end{table}

\begin{table}[t]
\centering
\caption{Lead-wise forecast skill for the remaining geopotential levels. The upper block reports latitude-weighted RMSE in metres and its relative reduction under WaveGSSM. The lower block reports centred anomaly correlation and the absolute gain $\Delta\mathrm{AC}$. WaveGSSM improves both metrics at every pressure level and lead shown.}
\label{tab:more_levels}
\setlength{\tabcolsep}{4.2pt}
\renewcommand{\arraystretch}{1.12}
\footnotesize
\resizebox{\textwidth}{!}{%
\begin{tabular}{@{}llcccccccccc@{}}
\toprule
\multirow{2}{*}{Level}
&
\multirow{2}{*}{Model}
&
\multicolumn{10}{c}{Lead time (days)}
\\
\cmidrule(lr){3-12}
&
&
1 & 2 & 3 & 4 & 5 & 6 & 7 & 8 & 9 & 10
\\
\midrule

\multicolumn{12}{@{}l}{\emph{RMSE (m)}~$\downarrow$}
\\[1pt]

\multirow{3}{*}{$1000$\,hPa}
& WaveGSSM
& 9.7 & 15.6 & 22.3 & 29.3 & 36.1 & 42.4 & 47.8 & 52.5 & 56.2 & 59.2
\\
& GraphCast$^{*}$
& 15.6 & 20.1 & 25.9 & 32.4 & 39.0 & 45.0 & 50.1 & 54.4 & 57.9 & 60.7
\\
& \impc{Reduction}
& \impc{37.7\%}
& \impc{22.1\%}
& \impc{14.0\%}
& \impc{9.8\%}
& \impc{7.4\%}
& \impc{5.7\%}
& \impc{4.7\%}
& \impc{3.6\%}
& \impc{2.8\%}
& \impc{2.5\%}
\\

\addlinespace[2pt]

\multirow{3}{*}{$850$\,hPa}
& WaveGSSM
& 9.1 & 14.9 & 21.4 & 28.1 & 34.8 & 41.0 & 46.4 & 51.1 & 54.8 & 57.7
\\
& GraphCast$^{*}$
& 15.1 & 19.3 & 24.9 & 31.2 & 37.6 & 43.4 & 48.5 & 52.7 & 56.1 & 58.9
\\
& \impc{Reduction}
& \impc{39.7\%}
& \impc{22.7\%}
& \impc{14.2\%}
& \impc{9.9\%}
& \impc{7.3\%}
& \impc{5.6\%}
& \impc{4.3\%}
& \impc{3.1\%}
& \impc{2.2\%}
& \impc{1.9\%}
\\

\addlinespace[2pt]

\multirow{3}{*}{$300$\,hPa}
& WaveGSSM
& 16.4 & 26.7 & 39.0 & 52.6 & 66.1 & 78.9 & 90.4 & 100.4 & 109.0 & 115.9
\\
& GraphCast$^{*}$
& 25.8 & 34.9 & 45.9 & 59.0 & 72.2 & 84.8 & 96.1 & 105.5 & 113.1 & 119.3
\\
& \impc{Reduction}
& \impc{36.2\%}
& \impc{23.4\%}
& \impc{14.9\%}
& \impc{10.9\%}
& \impc{8.5\%}
& \impc{7.0\%}
& \impc{5.9\%}
& \impc{4.8\%}
& \impc{3.7\%}
& \impc{2.9\%}
\\

\addlinespace[2pt]

\multirow{3}{*}{$150$\,hPa}
& WaveGSSM
& 14.2 & 23.5 & 33.0 & 43.3 & 53.9 & 64.1 & 73.5 & 82.1 & 89.5 & 95.7
\\
& GraphCast$^{*}$
& 23.4 & 33.0 & 42.2 & 53.0 & 64.4 & 75.5 & 85.8 & 94.8 & 102.3 & 108.6
\\
& \impc{Reduction}
& \impc{39.2\%}
& \impc{28.8\%}
& \impc{21.7\%}
& \impc{18.2\%}
& \impc{16.3\%}
& \impc{15.0\%}
& \impc{14.2\%}
& \impc{13.4\%}
& \impc{12.5\%}
& \impc{11.9\%}
\\

\midrule

\multicolumn{12}{@{}l}{\emph{Centred anomaly correlation}~$\uparrow$}
\\[1pt]

\multirow{3}{*}{$1000$\,hPa}
& WaveGSSM
& 0.984 & 0.959 & 0.914 & 0.848 & 0.763 & 0.667 & 0.569 & 0.475 & 0.390 & 0.319
\\
& GraphCast$^{*}$
& 0.959 & 0.932 & 0.882 & 0.810 & 0.715 & 0.609 & 0.503 & 0.407 & 0.323 & 0.252
\\
& \impc{$\Delta$AC}
& \impc{+0.025}
& \impc{+0.027}
& \impc{+0.032}
& \impc{+0.038}
& \impc{+0.048}
& \impc{+0.058}
& \impc{+0.066}
& \impc{+0.068}
& \impc{+0.067}
& \impc{+0.067}
\\

\addlinespace[2pt]

\multirow{3}{*}{$850$\,hPa}
& WaveGSSM
& 0.986 & 0.962 & 0.920 & 0.856 & 0.775 & 0.680 & 0.584 & 0.489 & 0.405 & 0.333
\\
& GraphCast$^{*}$
& 0.961 & 0.936 & 0.889 & 0.820 & 0.729 & 0.625 & 0.521 & 0.425 & 0.341 & 0.270
\\
& \impc{$\Delta$AC}
& \impc{+0.025}
& \impc{+0.026}
& \impc{+0.031}
& \impc{+0.036}
& \impc{+0.046}
& \impc{+0.055}
& \impc{+0.063}
& \impc{+0.064}
& \impc{+0.064}
& \impc{+0.063}
\\

\addlinespace[2pt]

\multirow{3}{*}{$300$\,hPa}
& WaveGSSM
& 0.988 & 0.968 & 0.931 & 0.873 & 0.797 & 0.708 & 0.613 & 0.518 & 0.429 & 0.355
\\
& GraphCast$^{*}$
& 0.971 & 0.947 & 0.906 & 0.843 & 0.759 & 0.662 & 0.561 & 0.466 & 0.385 & 0.316
\\
& \impc{$\Delta$AC}
& \impc{+0.017}
& \impc{+0.021}
& \impc{+0.025}
& \impc{+0.030}
& \impc{+0.038}
& \impc{+0.046}
& \impc{+0.052}
& \impc{+0.052}
& \impc{+0.044}
& \impc{+0.039}
\\

\addlinespace[2pt]

\multirow{3}{*}{$150$\,hPa}
& WaveGSSM
& 0.989 & 0.969 & 0.938 & 0.890 & 0.827 & 0.751 & 0.669 & 0.584 & 0.504 & 0.433
\\
& GraphCast$^{*}$
& 0.970 & 0.944 & 0.907 & 0.851 & 0.777 & 0.690 & 0.596 & 0.507 & 0.428 & 0.361
\\
& \impc{$\Delta$AC}
& \impc{+0.019}
& \impc{+0.025}
& \impc{+0.031}
& \impc{+0.039}
& \impc{+0.050}
& \impc{+0.061}
& \impc{+0.073}
& \impc{+0.077}
& \impc{+0.076}
& \impc{+0.072}
\\

\bottomrule
\end{tabular}%
}
\end{table}

\subsection{Retrained Ablation of Persistent Motion-State Carry}
\label{app:weather_retrained_no_carry}

We first test whether carrying the motion state remains useful when the ablated model is allowed to retrain and adapt. The parameter-matched variant is trained independently under the same data, objective, optimizer, rollout horizon, and validation-selection protocol as WaveGSSM. Before every internal graph-wave transition, including the first, its incoming motion state is set to zero. Each transition still computes an instantaneous motion response and uses it to advance the pattern state. The terminal instantaneous motion state also remains available to the decoder.

We use mean geopotential RMSE over all six pressure levels and forecast days~1--5 as the aggregate endpoint
\begin{equation}
\overline{\operatorname{RMSE}}_{\mathrm{geo},1\text{--}5}
=
\frac{1}{30}
\sum_{p\in\mathcal P}
\sum_{h=1}^{5}
\operatorname{RMSE}(p,h),
\qquad
\mathcal P
=
\{150,200,300,500,850,1000\}\,\mathrm{hPa}.
\label{eq:weather_retrained_geopotential_endpoint}
\end{equation}
We also report z200 at regularly sampled forecast leads because it is already used as the representative upper-tropospheric field in the main text.

\begin{table}[t]
\centering
\caption{\textbf{Retrained ablation of motion-state carry on ERA5.} Both variants are independently trained under the same protocol with identical parameter counts. The aggregate geopotential RMSE averages all six pressure levels over forecast days~1--5. All values are in metres and lower is better.}
\label{tab:weather_retrained_no_carry}
\setlength{\tabcolsep}{4.0pt}
\renewcommand{\arraystretch}{1.08}
\small
\begin{tabular}{@{}lcccccc@{}}
\toprule
\multirow{2}{*}{Model}
& \multirow{2}{*}{\shortstack[c]{Geopotential RMSE\\6 levels, D1--5}}
& \multicolumn{5}{c}{z200 RMSE} \\
\cmidrule(lr){3-7}
& & D1 & D3 & D5 & D7 & D10 \\
\midrule
WaveGSSM
& \textbf{30.72}
& \textbf{15.08}
& \textbf{35.77}
& \textbf{60.59}
& \textbf{83.36}
& \textbf{108.17} \\
\begin{tabular}[c]{@{}l@{}}
WaveGSSM w/o\\
motion-state carry
\end{tabular}
& 31.53
& 15.59
& 37.16
& 62.49
& 85.58
& 110.84 \\
\midrule
Relative increase
& $+2.63\%$
& $+3.40\%$
& $+3.90\%$
& $+3.13\%$
& $+2.66\%$
& $+2.46\%$ \\
\bottomrule
\end{tabular}
\end{table}

Removing motion-state carry increases the mean geopotential RMSE from $30.72$ to $31.53$\,m over all six pressure levels and forecast days~1--5, corresponding to a $2.63\%$ increase. The same direction holds at every reported z200 lead, with relative increases between $2.46\%$ and $3.90\%$.

The difference is modest but directionally consistent across the reported leads. Because the ablated model is retrained from initialization, its remaining pathways can adapt to the removed motion-state carry. The remaining performance gap indicates that carrying the motion state provides additional predictive value within the proposed architecture.

\subsection{Frozen-Checkpoint Motion-State Interventions}
\label{app:weather_velocity_ablation}

The retrained comparison in \Cref{tab:weather_retrained_no_carry} evaluates motion-state carry after the ablated model has adapted through training. We next examine how the selected full predictor uses this pathway at inference time. We freeze the validation-selected WaveGSSM checkpoint and separately intervene on the motion state entering the internal graph-wave transitions and the terminal motion state supplied to the decoder.

\paragraph{Intervention semantics.}
Let $\alpha_{\mathrm c}$ scale the motion state entering the internal graph-wave transitions and $\alpha_{\mathrm r}$ scale the terminal motion state supplied to the decoder.
Before every internal transition, including the first, the incoming motion state is replaced by
\begin{equation}
\widetilde{\m V}_{t-1}^{(\ell)}
=
\alpha_{\mathrm c}
\m V_{t-1}^{(\ell)}.
\label{eq:weather_carry_intervention}
\end{equation}
The transition then computes its instantaneous motion response and advances the pattern state as usual.
At the decoder, the state readout becomes
\begin{equation}
\widetilde{\m H}_{\mathrm{state}}
=
\left[
\m U_T^{(L)},
\alpha_{\mathrm r}\m V_T^{(L)}
\right].
\label{eq:weather_readout_intervention}
\end{equation}

Setting $\alpha_{\mathrm c}=0$ suppresses both the initialized motion cue and the motion state carried between subsequent internal transitions, while preserving the instantaneous response computed by each transition. Setting $\alpha_{\mathrm r}=0$ removes only the terminal motion state from the decoder input.
Neither intervention changes the model parameters, graph structure, forcing, or tensor dimensions.

The explicit wave state is propagated only within each graph-wave processor call.
Across forecast days, predicted prognostic fields form the next rolling input history, from which a new wave state is reconstructed.
An explicit $(\m U,\m V)$ pair is not passed directly between forecast days.

\paragraph{Design and endpoint.}
All conditions use the same checkpoint selected only on 2024 validation data. They are evaluated on the same ordered 354 initializations from the held-out year 2025, all eight output channels, and lead days 1--10. The frozen $2\times2$ factorial evaluates
\begin{equation}
(\alpha_{\mathrm c},\alpha_{\mathrm r})
\in
\left\{
(1,1),
(0,1),
(1,0),
(0,0)
\right\}.
\end{equation}
The continuous analysis fixes one scale at one while sweeping the other over $\{1,.75,.5,.25,0\}$.
The scale grid was fixed before inspecting the test results, and its endpoints coincide with the corresponding factorial conditions.

For intervention condition
$m$,
initialization
$i$,
lead
$h$,
and output channel
$q$,
let
$\operatorname{wMSE}_{m,i,h,q}$
be the cosine-latitude-weighted mean squared error and let
$s_q$
be the fixed training scale of channel
$q$
stored with the checkpoint runtime.
The primary dimensionless endpoint is
\begin{equation}
E_m
=
\sqrt{
\frac{1}{NHQ}
\sum_{i=1}^{N}
\sum_{h=1}^{H}
\sum_{q=1}^{Q}
\frac{
\operatorname{wMSE}_{m,i,h,q}
}{
s_q^2
}
},
\qquad
\Delta_m
=
100
\left(
\frac{E_m}{E_{\mathrm{Full}}}
-
1
\right),
\label{eq:weather_velocity_endpoint}
\end{equation}
where
$N=354$,
$H=10$,
and
$Q=8$.
Soil moisture is evaluated over land.
Positive
$\Delta_m$
means that the intervention increases forecast error.

\begin{figure}[t]
\centering
\includegraphics[width=\linewidth]{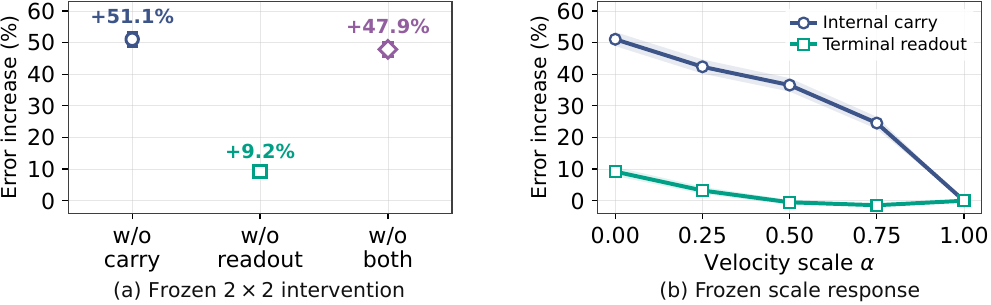}
\caption{Frozen-checkpoint functional reliance on the motion-state pathway.
All conditions use the same validation-selected checkpoint and 354 held-out 2025 ERA5 initializations without retraining.
Vertical axes report percentage changes in aggregate standardized error relative to Full over eight fields and forecast days~1--10.
(a) Non-Full conditions in the internal-pathway $\times$ readout design.
(b) Error response to scaling the internal motion-state pathway or terminal motion-state readout while keeping the other pathway unchanged.
The plotted ``w/o carry'' condition corresponds to $\alpha_{\mathrm c}=0$ and therefore suppresses the motion state entering all internal transitions, including the initialized motion state.
}
\label{fig:velocity_interventions}
\end{figure}

\paragraph{Factorial results.}
\Cref{tab:weather_velocity_factorial} reports the complete $2\times2$ design.
Suppressing the internal motion-state pathway produces the largest degradation. Removing terminal readout alone has a smaller effect, while joint removal is less harmful than the sum of the two individual effects, indicating that the two pathways are not simply additive.
\begin{table}[t]
\centering
\caption{Frozen-checkpoint motion-state intervention on ERA5. Values report the percentage change in aggregate standardized error relative to Full.}
\label{tab:weather_velocity_factorial}
\setlength{\tabcolsep}{6pt}
\small
\begin{tabular}{@{}lccc@{}}
\toprule
Condition
& Internal path
& Readout
& Error change (\%) \\
\midrule
Full
& On
& On
& $0.0$ \\
Internal path off
& Off
& On
& $+51.1$ \\
Readout off
& On
& Off
& $+9.2$ \\
Both off
& Off
& Off
& $+47.9$ \\
\bottomrule
\end{tabular}
\end{table}

\paragraph{Continuous scale response.}
\Cref{tab:weather_velocity_scale}
shows a graded response to the internal motion-state scale.
Reducing
$\alpha_{\mathrm c}$
consistently increases the aggregate endpoint.
The degradation at
$\alpha_{\mathrm c}=0$
is therefore part of a continuous sensitivity pattern rather than an isolated effect of complete removal.

Terminal readout behaves differently.
A scale of
$.75$
slightly improves the frozen endpoint,
$.5$
remains close to Full, and stronger attenuation becomes harmful.
This non-monotonic response suggests that the terminal motion state acts partly through feature calibration rather than simple dependence on its magnitude.
Because the scales are test-time interventions on the held-out year, the favorable
$.75$
condition is treated only as a diagnostic.

\begin{table}[t]
\centering
\caption{Motion-state scale sweep on the frozen ERA5 checkpoint. Values report percentage changes in aggregate standardized error relative to Full, with the other pathway kept at full scale. Lower is better.}
\label{tab:weather_velocity_scale}
\setlength{\tabcolsep}{7pt}
\small
\begin{tabular}{@{}ccc@{}}
\toprule
Scale $\alpha$
& \shortstack{Internal path\\error change (\%)}
& \shortstack{Terminal readout\\error change (\%)} \\
\midrule
$1.00$
& $0.0$
& $0.0$ \\
$0.75$
& $+24.6$
& $-1.4$ \\
$0.50$
& $+36.6$
& $-0.5$ \\
$0.25$
& $+42.4$
& $+3.2$ \\
$0.00$
& $+51.1$
& $+9.2$ \\
\bottomrule
\end{tabular}
\end{table}

\paragraph{Structure of the intervention effect.}
\Cref{fig:weather_velocity_detail} shows that suppressing the internal motion-state pathway is harmful across the full ten-day horizon and across geopotential, temperature, and soil moisture.
The relative effect is largest at short lead, where the Full forecast error is smallest.
This trend should not be interpreted as accumulation of an explicit motion state across forecast days, since each processor call reconstructs its wave state from the rolling weather history.

For z500, suppressing the internal pathway degrades both phase coherence and amplitude fidelity over days~3--7 and zonal wavenumbers~4--9.
The larger spectral contribution comes from amplitude error, so we do not attribute the intervention effect primarily to phase.

\begin{figure*}[t]
\centering
\includegraphics[width=\textwidth]{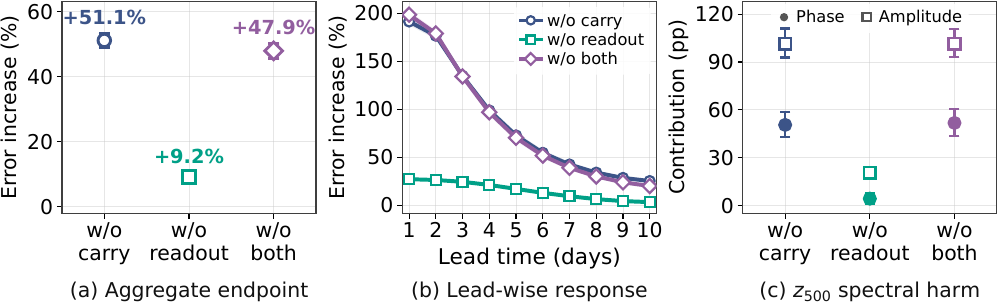}
\caption{Frozen-checkpoint motion-state diagnostics. (a) Aggregate standardized error increase relative to Full for the three non-Full interventions over eight fields and forecast days~1--10. (b) Lead-wise error increase relative to Full. (c) Phase and amplitude contributions to z500 spectral error over days~3--7 and zonal wavenumbers~4--9, normalized by the Full spectral error and reported in percentage points. Suppressing the internal motion-state pathway increases error across forecast leads and affects both spectral components, with the larger contribution coming from amplitude.}
\label{fig:weather_velocity_detail}
\end{figure*}

The phase-specific damage under $\alpha_{\mathrm c}=0$ increases with initialization wave activity (\Cref{fig:weather_velocity_regime}a). A season-adjusted linear fit gives a slope of $10.37$ percentage points per activity standard deviation, consistent with a stronger phase effect in more wave-active conditions. This relationship is specific to the phase component and does not imply the same activity dependence for total forecast error.

\begin{figure*}[t]
\centering
\includegraphics[width=\textwidth]{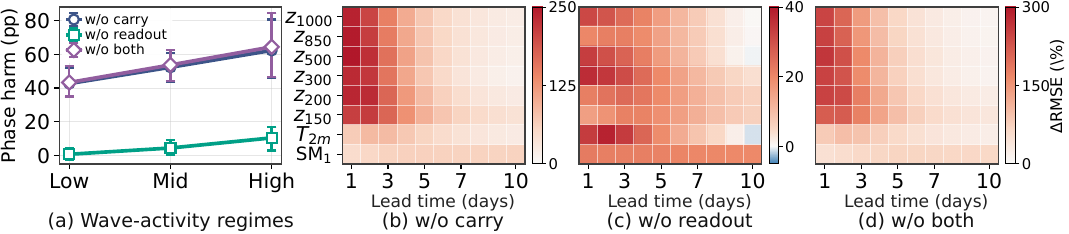}
\caption{Regime and variable detail for the frozen factorial. (a) Phase-harm contribution over days~3--7 and zonal wavenumbers~4--9, normalized by Full total spectral error and expressed in percentage points, stratified by initialization wave activity using thresholds fixed from the 2024 validation set. (b)--(d) Physical-RMSE increase relative to Full by target field and lead time under suppression of the internal motion-state pathway, removal of terminal readout, and joint removal. Each heatmap uses its own colorbar to show within-condition structure. Conditions without the internal motion-state pathway show broader degradation, while readout-only effects are smaller.}
\label{fig:weather_velocity_regime}
\end{figure*}

\FloatBarrier

\section{Atmospheric Forecast Diagnostics}
\label{app:additional_era5_diagnostics}

We next examine the spatial and structural character of the ERA5 forecast improvement. All diagnostics use the same 354 held-out 2025 initializations and the evaluation protocol in \Cref{app:setup}. We examine whether the gain is geographically broad, whether it extends across atmospheric levels, whether it can be explained by mean-bias correction, and whether the forecasts retain synoptic-scale variability.

\subsection{Additional Regime and Spectral Detail}
\label{app:analysis_detail}

\Cref{fig:analysis_appendix} provides the supporting breakdowns omitted from the focused main-text analysis.
The seasonal and regional views show that the z500 anomaly-correlation advantage is not confined to one initialization season or geographical region, while the day-7 spectrum resolves how power is distributed within and around the selected synoptic band.

\begin{figure}[th]
\centering
\includegraphics[width=\linewidth]{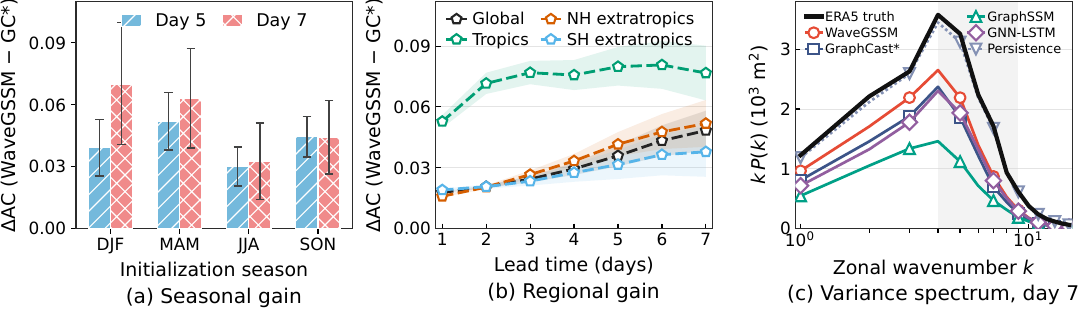}
\caption{Additional seasonal, regional, and spectral detail for the z500 forecast advantage. Panels (a) and (b) report $\Delta\mathrm{AC}=\mathrm{AC}_{\mathrm{WaveGSSM}}-\mathrm{AC}_{\mathrm{GC}^{*}}$ relative to GraphCast$^{*}$. Panel (a) groups the Northern-Hemisphere extratropical gain at days~5 and~7 by initialization season, where DJF, MAM, JJA, and SON denote December--February, March--May, June--August, and September--November. Panel (b) reports the gain by region and forecast lead. Panel (c) shows the day-7 area-weighted z500 spectrum, with zonal wavenumbers~4--9 shaded.}
\label{fig:analysis_appendix}
\end{figure}

\subsection{Horizontal Spatial Skill}
\label{app:spatial_skill_maps}

Global RMSE does not reveal whether a forecast improvement is spatially broad. We therefore evaluate local skill at every valid grid point.

For variable $v$, lead $h$, and grid point
$x=(\lambda,\varphi)$,
let
$\operatorname{RMSE}_{M}(v,h,x)$
denote the RMSE of model $M$ over the 354 test initializations. We define
\begin{equation}
S_{v,h}(x)
=
\frac{
\operatorname{RMSE}_{\mathrm{WaveGSSM}}(v,h,x)
-
\operatorname{RMSE}_{\mathrm{GC}^{*}}(v,h,x)
}{
\operatorname{RMSE}_{\mathrm{GC}^{*}}(v,h,x)
}.
\label{eq:local_normalized_rmse_gap}
\end{equation}
Negative values favor WaveGSSM and positive values favor GraphCast$^{*}$.
The maps use blue and red accordingly.
We also report the fraction of valid grid points with
$S_{v,h}(x)<0$
as a measure of the spatial coverage of the improvement.

The geopotential advantage is geographically broad. At day~1, WaveGSSM reduces global RMSE by $39.1\%$ for $z200$ and $38.1\%$ for $z500$, while achieving lower local RMSE at more than $99\%$ of valid grid points for both fields.
The advantage remains widespread as forecast lead increases.
At day~7, WaveGSSM has lower local RMSE at $94.6\%$ of valid $z200$ grid points and $81.8\%$ of valid $z500$ grid points.

The surface variables show a weaker and less uniform spatial advantage.
WaveGSSM improves 2\,m daily maximum temperature and soil moisture layer~1 over much of the domain, but the gains are less extensive than for upper-air geopotential.
Together, the maps show that the geopotential improvement is spatially broad rather than concentrated in a small number of favorable regions.

\begin{figure}[th]
\centering
\includegraphics[width=\linewidth]{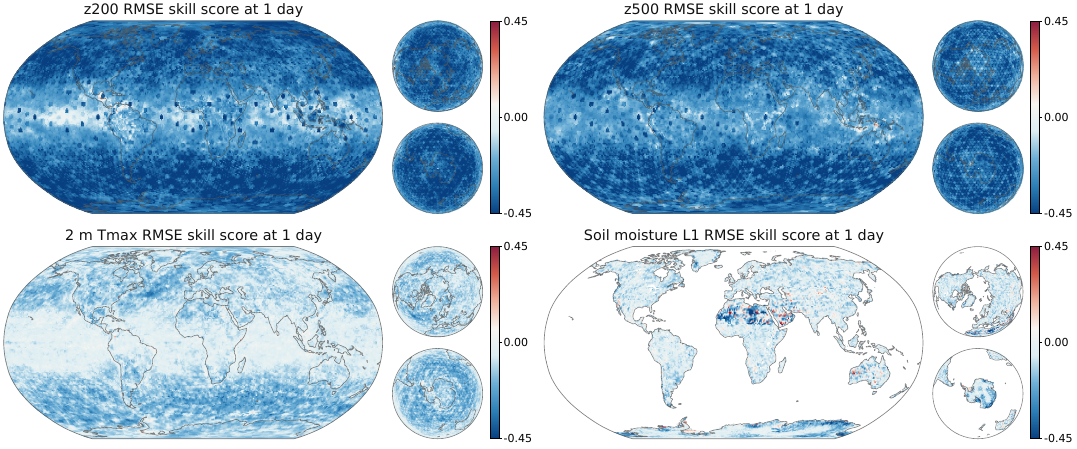}
\caption{Spatial RMSE skill relative to GraphCast$^{*}$ at lead day~1. Each panel shows the normalized local RMSE difference in Eq.~(\ref{eq:local_normalized_rmse_gap}) over the 2025 test initializations. Blue indicates lower WaveGSSM RMSE, while red indicates lower GraphCast$^{*}$ RMSE. The variables are $z200$, $z500$, 2\,m daily maximum temperature, and soil moisture layer~1. Polar insets display the high-latitude regions. The short-range geopotential improvement is broad rather than concentrated in a small number of locations.}
\label{fig:exp06_spatial_skill_lead01}
\end{figure}

\begin{figure}[th]
\centering
\includegraphics[width=\linewidth]{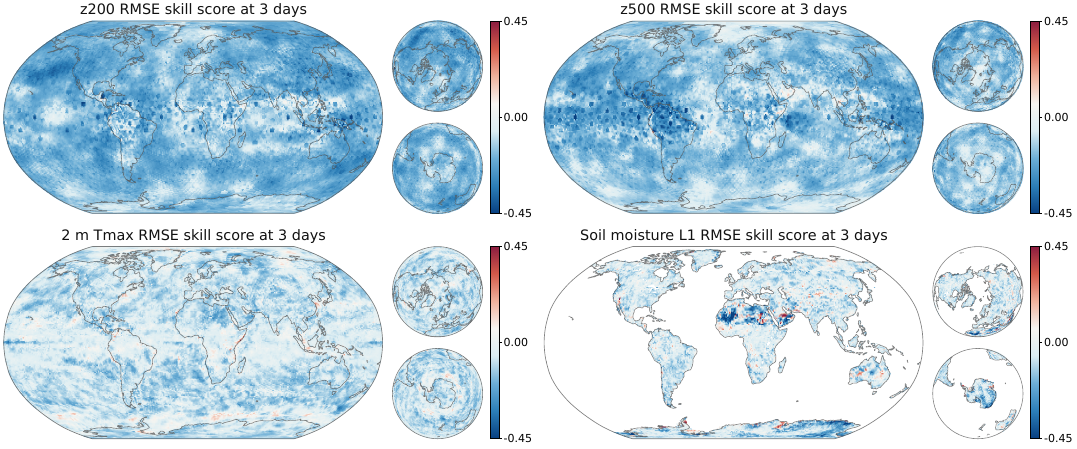}
\caption{
Spatial RMSE skill relative to GraphCast$^{*}$ at lead day~3.
The score and color convention follow
\Cref{fig:exp06_spatial_skill_lead01}.
WaveGSSM reduces global RMSE by
$18.6\%$
for $z200$ and
$14.1\%$
for $z500$,
and the corresponding local improvement remains geographically widespread.
}
\label{fig:exp06_spatial_skill_lead03}
\end{figure}

\begin{figure}[th]
\centering
\includegraphics[width=\linewidth]{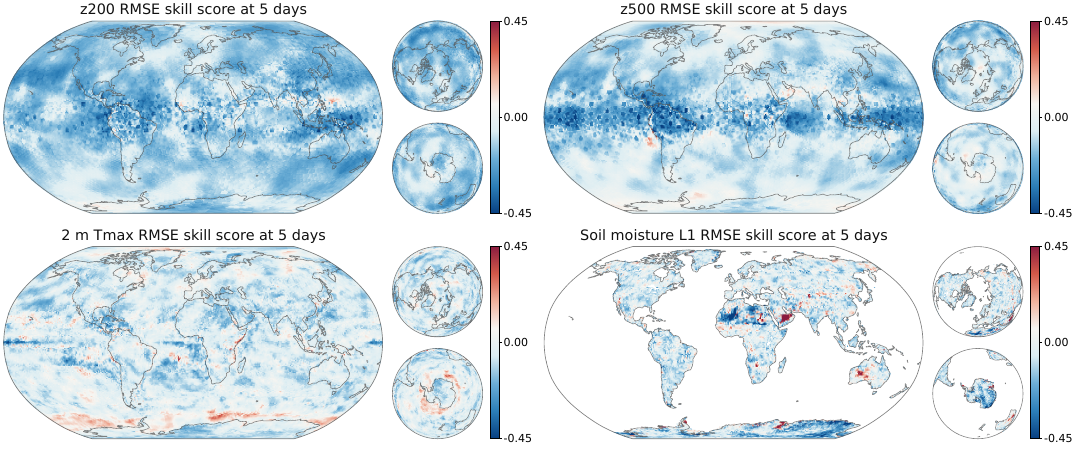}
\caption{
Spatial RMSE skill relative to GraphCast$^{*}$ at lead day~5.
Blue denotes lower local RMSE for WaveGSSM under the normalized score in
Eq.~(\ref{eq:local_normalized_rmse_gap}).
The strongest and most continuous improvements occur in the geopotential fields, while the surface-variable gains are more regionally heterogeneous.
}
\label{fig:exp06_spatial_skill_lead05}
\end{figure}

\begin{figure}[th]
\centering
\includegraphics[width=\linewidth]{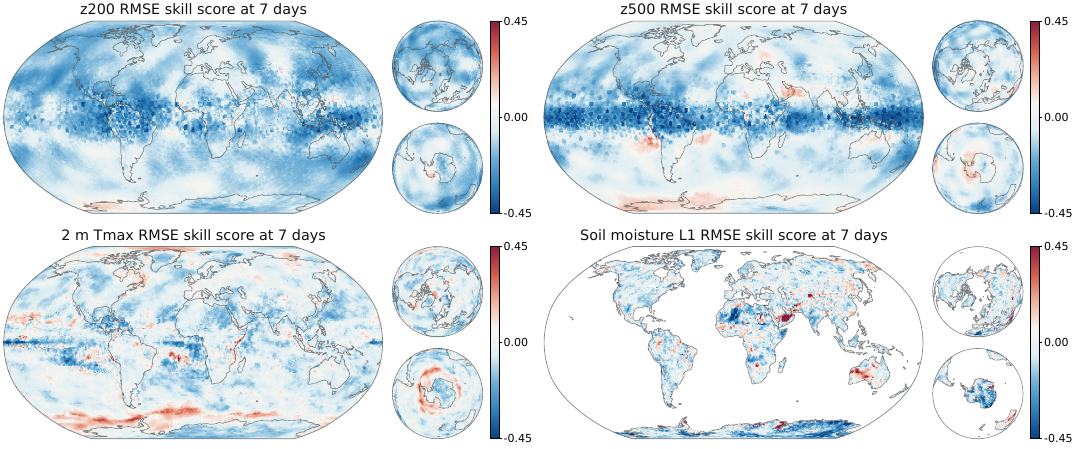}
\caption{
Spatial RMSE skill relative to GraphCast$^{*}$ at lead day~7.
The score and color convention follow the earlier lead-time maps.
WaveGSSM has lower local RMSE at
$94.6\%$
of valid $z200$ grid points and
$81.8\%$
of valid $z500$ grid points.
The pointwise margin is smaller than at short range, but the upper-air advantage remains spatially broad after one week.
}
\label{fig:exp06_spatial_skill_lead07}
\end{figure}

\subsection{Vertical Structure of Geopotential Skill}
\label{app:vertical_skill}

The horizontal maps show that the geopotential improvement is geographically broad. We next examine how this advantage varies across latitude and pressure level.

For model $M$, lead $h$, latitude $\varphi$, and pressure level $p$, we compute RMSE over longitude and the test initializations
\begin{equation}
\operatorname{RMSE}_{M}(h,\varphi,p)
=
\left[
\frac{
1
}{
N_{\mathrm{test}} I
}
\sum_{n=1}^{N_{\mathrm{test}}}
\sum_{i=1}^{I}
\left(
\hat z_{M,n}(h,\lambda_i,\varphi,p)
-
z_{n}(h,\lambda_i,\varphi,p)
\right)^2
\right]^{1/2},
\label{eq:latitude_pressure_rmse}
\end{equation}
where
$N_{\mathrm{test}}=354$
and
$I=360$.
We then define the normalized skill difference
\begin{equation}
S_h(\varphi,p)
=
\frac{
\operatorname{RMSE}_{\mathrm{WaveGSSM}}(h,\varphi,p)
-
\operatorname{RMSE}_{\mathrm{GC}^{*}}(h,\varphi,p)
}{
\operatorname{RMSE}_{\mathrm{GC}^{*}}(h,\varphi,p)
}.
\label{eq:vertical_normalized_rmse_gap}
\end{equation}
Negative values favor WaveGSSM.

\Cref{fig:exp09_vertical_geopotential_skill} shows that the improvement extends across the predicted pressure levels rather than being confined to a single field.
The strongest and most persistent gains appear in the upper-level geopotential fields, particularly around $200$ and $300$\,hPa. This pattern is consistent with the lead-wise results, where the upper-level fields retain some of the largest medium-range advantages.

The latitude--pressure structure is also coherent across broad regions rather than appearing as isolated cells.
Together with the horizontal maps, this shows that the geopotential improvement extends across both geographical location and atmospheric level.

\begin{figure}[th]
\centering
\includegraphics[width=\linewidth]{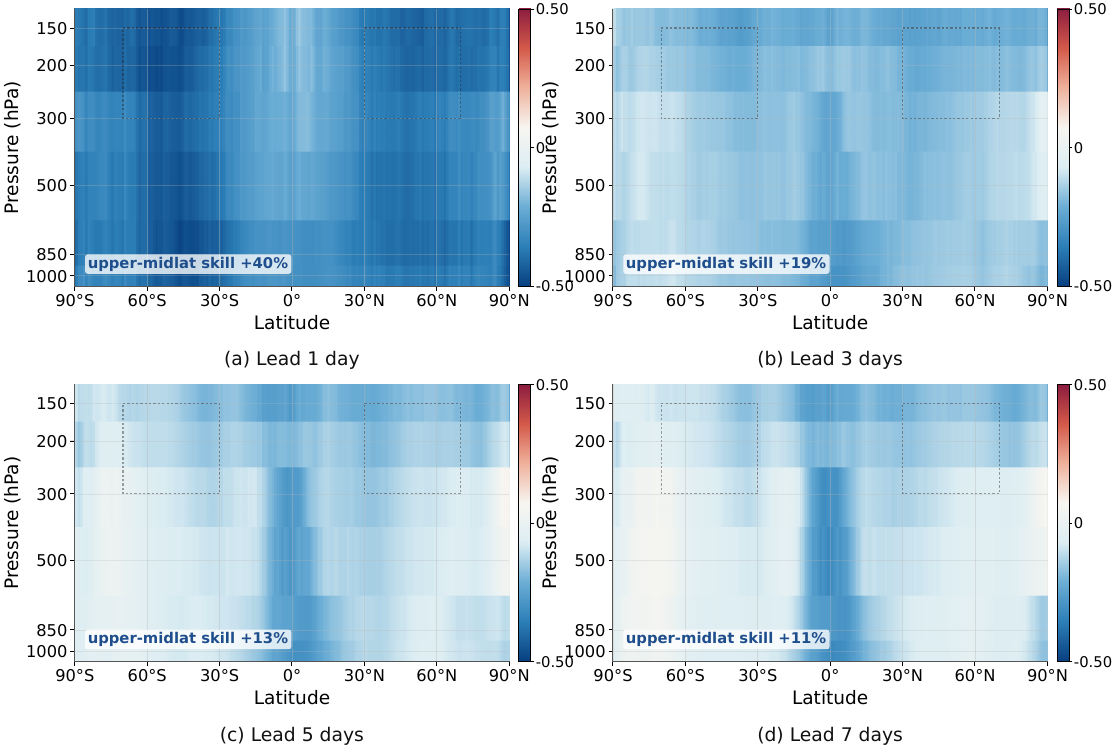}
\caption{Latitude--pressure structure of geopotential skill. Each panel shows the normalized RMSE difference in Eq.~(\ref{eq:vertical_normalized_rmse_gap}) over longitude and the 2025 test initializations. Blue indicates lower WaveGSSM RMSE and red indicates lower GraphCast$^{*}$ RMSE. The advantage extends across multiple pressure levels and remains strongest in the upper-level geopotential fields.}
\label{fig:exp09_vertical_geopotential_skill}
\end{figure}

\subsection{Systematic Bias Analysis}
\label{app:bias_maps}

Lower RMSE can reflect better transient forecasts, smaller systematic bias, or both. We therefore examine the mean forecast bias separately from the total error.

For model $M$, variable $v$, lead $h$, and grid point $x=(\lambda,\varphi)$, we define
\begin{equation}
B_M(v,h,x)
=
\frac{1}{N_{\mathrm{test}}}
\sum_{n=1}^{N_{\mathrm{test}}}
\left[
\hat y_{M,n}(v,h,x)
-
y_n(v,h,x)
\right],
\qquad
N_{\mathrm{test}}=354.
\label{eq:spatial_mean_bias}
\end{equation}
Positive values indicate forecasts above ERA5 on average, while negative values indicate forecasts below ERA5.

To compare bias magnitude between models, we use
\begin{equation}
D_{\mathrm{bias}}(v,h,x)
=
\left|
B_{\mathrm{WaveGSSM}}(v,h,x)
\right|
-
\left|
B_{\mathrm{GC}^{*}}(v,h,x)
\right|.
\label{eq:absolute_bias_gap}
\end{equation}
Negative values favor WaveGSSM.
We summarize each map using the latitude-weighted absolute bias
\begin{align}
\operatorname{ABias}_M(v,h)
&=
\frac{
\sum_x
w_x
\left|
B_M(v,h,x)
\right|
}{
\sum_x w_x
},
\qquad
w_x=\cos\varphi_x,
\\
\operatorname{Reduction}_{\mathrm{bias}}
&=
\frac{
\operatorname{ABias}_{\mathrm{GC}^{*}}
-
\operatorname{ABias}_{\mathrm{WaveGSSM}}
}{
\operatorname{ABias}_{\mathrm{GC}^{*}}
}
\times 100\%.
\label{eq:relative_bias_reduction}
\end{align}
For soil moisture, the sums include land points only.

WaveGSSM substantially reduces systematic bias in the geopotential fields. For $z200$, the latitude-weighted absolute bias is reduced by $61.6\%$, $62.4\%$, and $59.3\%$ at days~3, 5, and~7.
The corresponding reductions for $z500$ are $35.0\%$, $30.9\%$, and $27.3\%$.
The improvement therefore persists across the medium-range forecast window.

The surface fields show a weaker effect.
At day~7, absolute bias is reduced by $12.1\%$ for 2\,m daily maximum temperature, while soil moisture changes little.
The strongest bias reductions therefore occur in the geopotential fields, where the overall forecast advantage is also most pronounced.

\begin{figure}[th]
\centering
\includegraphics[width=\linewidth]{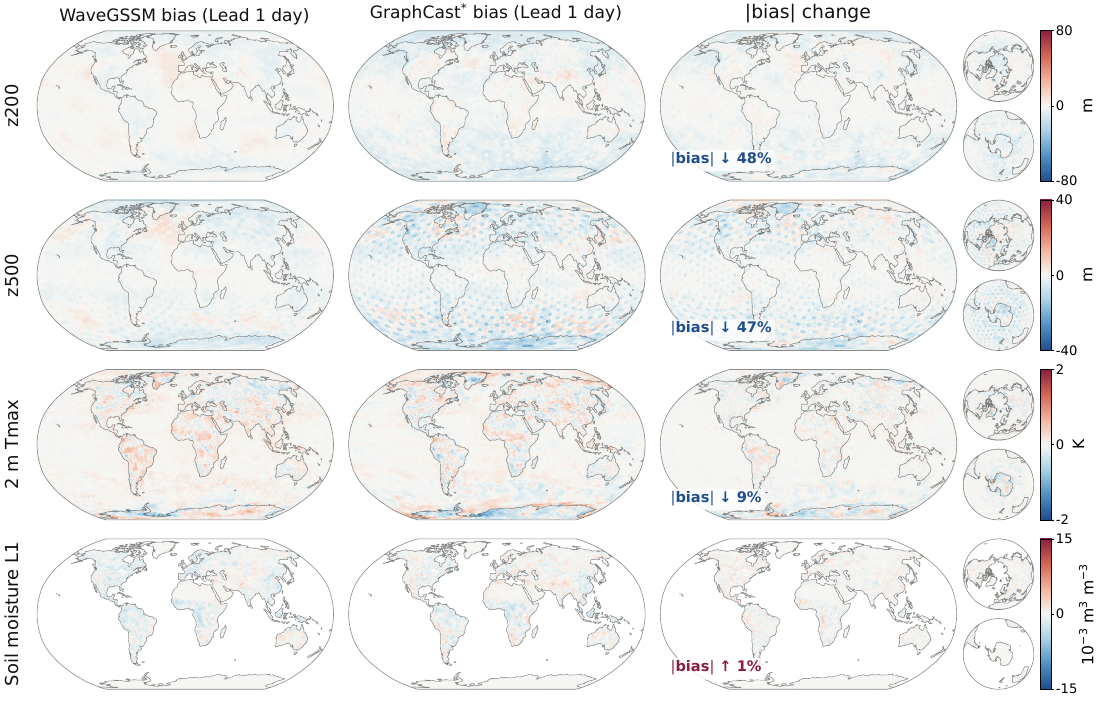}
\caption{Systematic forecast bias at lead day~1. The first two columns show signed mean bias for WaveGSSM and GraphCast$^{*}$. The third column shows the bias-magnitude difference in Eq.~(\ref{eq:absolute_bias_gap}). Blue indicates smaller bias magnitude for WaveGSSM and red indicates smaller bias magnitude for GraphCast$^{*}$.}
\label{fig:exp08_bias_maps_lead01}
\end{figure}

\begin{figure}[th]
\centering
\includegraphics[width=\linewidth]{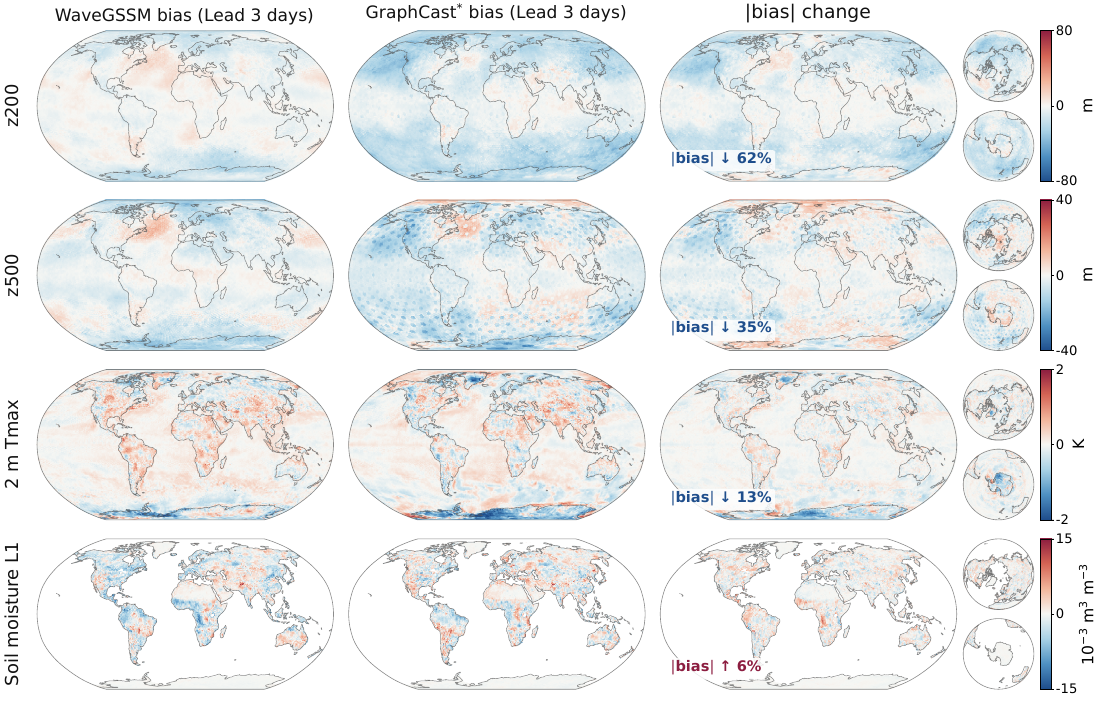}
\caption{Systematic forecast bias at lead day~3. WaveGSSM reduces latitude-weighted absolute bias by $61.6\%$ for $z200$ and $35.0\%$ for $z500$. The surface-field changes are smaller.}
\label{fig:exp08_bias_maps_lead03}
\end{figure}

\begin{figure}[th]
\centering
\includegraphics[width=\linewidth]{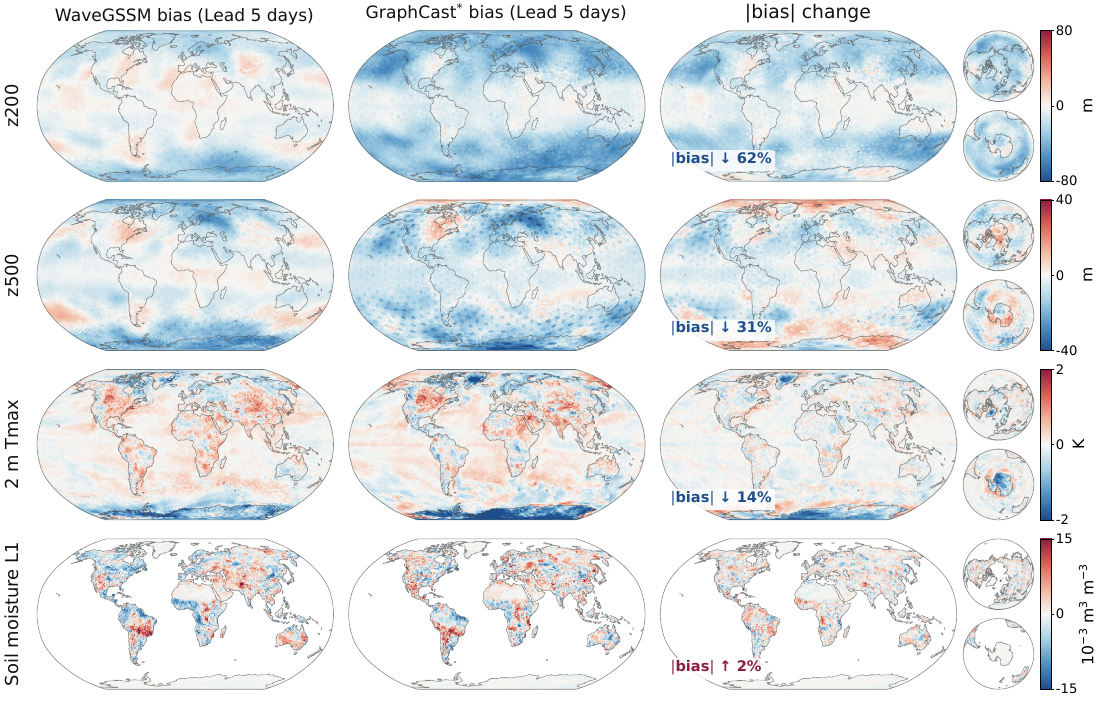}
\caption{Systematic forecast bias at lead day~5. WaveGSSM reduces latitude-weighted absolute bias by $62.4\%$ for $z200$ and $30.9\%$ for $z500$. The geopotential bias reduction remains broad at medium range.}
\label{fig:exp08_bias_maps_lead05}
\end{figure}

\begin{figure}[th]
\centering
\includegraphics[width=\linewidth]{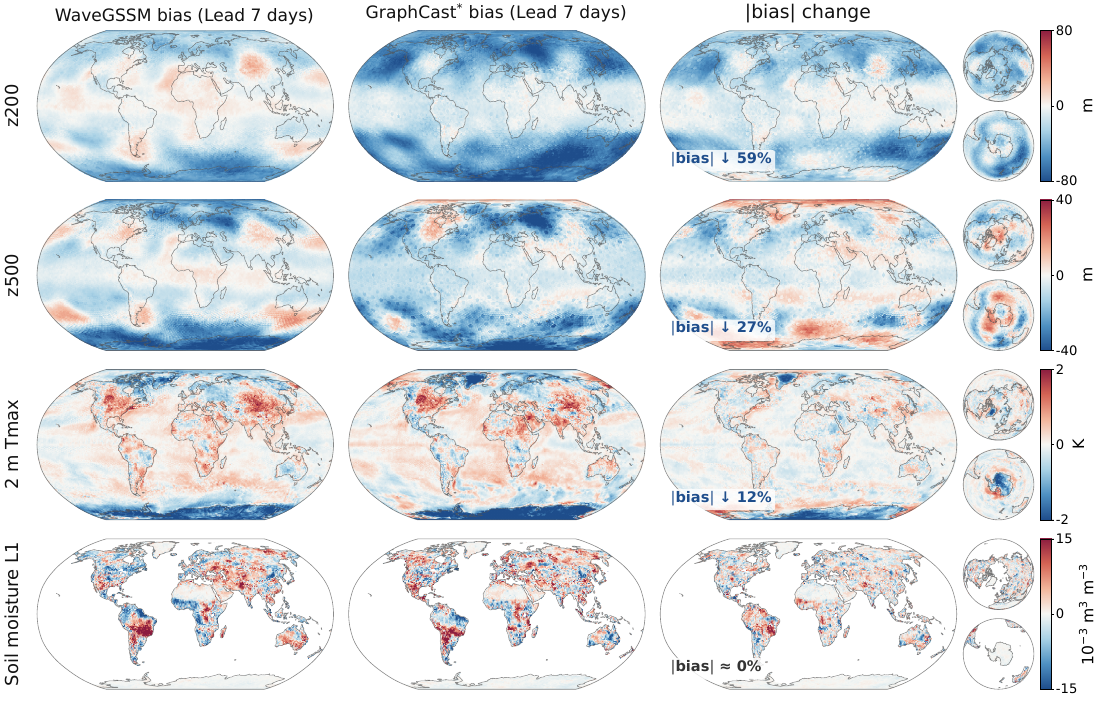}
\caption{Systematic forecast bias at lead day~7. WaveGSSM reduces latitude-weighted absolute bias by $59.3\%$ for $z200$, $27.3\%$ for $z500$, and $12.1\%$ for 2\,m daily maximum temperature. Soil-moisture bias changes little.}
\label{fig:exp08_bias_maps_lead07}
\end{figure}

\subsection{Synoptic-Scale Spectral Retention}
\label{app:spectral}

Lower squared error can also be obtained by damping spatial variability toward a smoother forecast
\citep{lam2023graphcast,price2025gencast}.
We therefore examine whether WaveGSSM preserves the spatial scales associated with evolving weather systems rather than obtaining lower error through stronger smoothing.

For each initialization and forecast lead, we compute geopotential anomalies relative to the 1979--2023 ERA5 day-of-year climatology
\begin{equation}
a(\varphi_j,\lambda_i)
=
f(\varphi_j,\lambda_i)
-
c(\varphi_j,\lambda_i),
\label{eq:spectral_anomaly}
\end{equation}
where
$f$
denotes either a forecast or the corresponding ERA5 field.
Geopotential is converted to geopotential height before spectral analysis.

At each latitude
$\varphi_j$,
we compute the discrete Fourier transform along the
$I=360$
longitude points
\begin{equation}
\widehat a_j(k)
=
\sum_{i=0}^{I-1}
a(\varphi_j,\lambda_i)
\exp
\left(
-\frac{2\pi\mathrm{i}ki}{I}
\right),
\qquad
k=0,\ldots,\left\lfloor\frac{I}{2}\right\rfloor .
\label{eq:zonal_dft}
\end{equation}
The corresponding one-sided Parseval-normalized power is
\begin{equation}
P_j(k)
=
\frac{m(k)}{I^2}
\left|
\widehat a_j(k)
\right|^2,
\label{eq:zonal_power}
\end{equation}
where
$m(k)=2$
for positive non-Nyquist wavenumbers and
$m(k)=1$
for the zero and Nyquist components.
This normalization makes the summed spectrum equal to the zonal anomaly variance.

For initialization $n$, we aggregate the spectra across latitude using cosine weights
\begin{equation}
P_n(k)
=
\frac{
\sum_j
\cos\varphi_j
P_{n,j}(k)
}{
\sum_j
\cos\varphi_j
}.
\label{eq:latitude_weighted_spectrum}
\end{equation}
The spectrum shown in \Cref{fig:analysis_appendix}c is obtained by averaging $P_n(k)$ over the $354$ test initializations.

We use zonal wavenumbers $4$--$9$ as an operational synoptic-scale band.
For each initialization, the retained variance ratio is
\begin{equation}
r_n
=
\frac{
\sum_{k=4}^{9}
P_n^{\mathrm{model}}(k)
}{
\sum_{k=4}^{9}
P_n^{\mathrm{ERA5}}(k)
},
\qquad
R_{\mathcal B}
=
\frac{1}{N_{\mathrm{test}}}
\sum_{n=1}^{N_{\mathrm{test}}}
r_n.
\label{eq:mean_synoptic_retention}
\end{equation}
A value of one indicates that the forecast retains the same amount of variance as ERA5 in this band.
The retention ratio measures amplitude at the selected scales, while anomaly correlation measures agreement in the spatial anomaly pattern.

\Cref{fig:analysis}b shows $R_{\mathcal B}$ across forecast lead for z500. At day~7, WaveGSSM retains $68\%$ of the ERA5 synoptic-scale variance, compared with $59\%$ for GraphCast$^{*}$ and $37\%$ for GraphSSM. \Cref{fig:analysis_appendix}c shows the corresponding day-7 spectrum, where WaveGSSM remains closest to ERA5 across much of the selected band.

Persistence retains nearly all of its initial spatial variance but has substantially lower forecast skill as weather systems evolve.
Variance retention alone is therefore not sufficient.
Together with the anomaly-correlation results, the spectral analysis shows that WaveGSSM retains more synoptic-scale variability while maintaining better agreement with the observed atmospheric pattern.

\section{Additional Forecast Analysis}
\label{app:additional_forecast_analysis}

\subsection{Auditable Forecast Case}
\label{app:case_study}

Aggregate scores establish that WaveGSSM improves pattern placement on average, but a case study makes the failure mode visually concrete. The case in \Cref{fig:exp07_z500_case_evolution} is selected by an auditable rule rather than by hand. Among 2025 Northern-Hemisphere extratropical z500 forecasts, the selected initialization must have strong ERA5 wave amplitude, a positive day~7 anomaly-correlation gain for WaveGSSM relative to GraphCast$^{*}$, and a lower WaveGSSM RMSE. This ranking rule avoids a common pitfall in qualitative forecast figures, where an example can look convincing but be difficult to reproduce.

The plotted quantity is the z500 zonal anomaly. This removes the dominant meridional background gradient and exposes the ridge and trough phase. The selected forecast is initialized on 2025-03-29 and verified at day~7 on 2025-04-05. The numerical evolution is consistent with the visual story. At day~1, both learned forecasts are accurate, but WaveGSSM already reduces RMSE by $37.9\%$ relative to GraphCast$^{*}$. At day~3 the RMSE reduction is $41.7\%$. At day~5 it is $30.6\%$. By day~7, GraphCast$^{*}$ has anomaly correlation $0.21$, while WaveGSSM reaches $0.62$ and lowers RMSE by $27.5\%$. This is the failure mode from the introduction in a real forecast. A snapshot model can retain a plausible-looking pattern but place it poorly. The motion-aware rollout keeps the same large-scale wave structure more closely aligned with ERA5.

\begin{figure}[th]
\centering
\includegraphics[width=\linewidth]{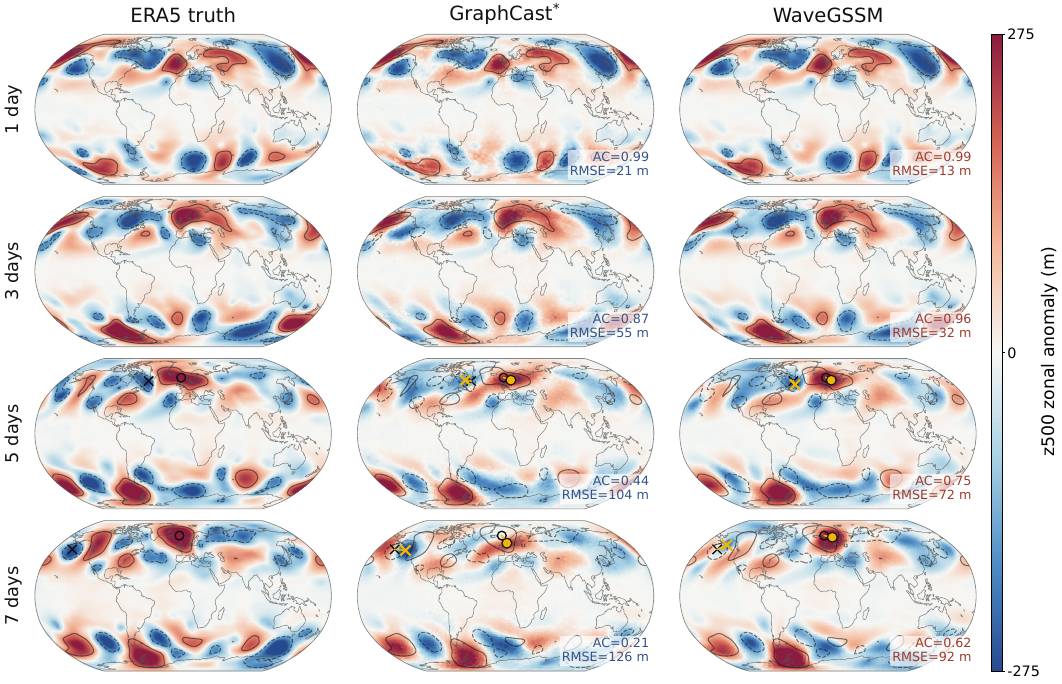}
\caption{Auditable z500 forecast case showing phase alignment over a one-week rollout. Rows are lead days~1, 3, 5, and~7. Columns are ERA5 truth, GraphCast$^{*}$, and WaveGSSM. Shading shows the 500\,hPa geopotential-height zonal anomaly (m). Black contours and markers trace the ERA5 ridge and trough reference, while yellow markers and segments show predicted centres and displacement. The case is initialized on 2025-03-29 and selected by the ranking rule described in the text. At day~7, WaveGSSM raises anomaly correlation from $0.21$ to $0.62$ and reduces RMSE by $27.5\%$ relative to GraphCast$^{*}$, illustrating the same phase-preservation effect quantified in the aggregate diagnostics.}
  \label{fig:exp07_z500_case_evolution}
\end{figure}

\subsection{Extratropical Skill Across Model Families}
\label{app:waveguide_skill}

The preceding results compare WaveGSSM primarily with the backbone-matched GraphCast$^{*}$ model over the global domain. We next ask whether the advantage remains visible in the extratropics and against other temporal model families.

We define the evaluation region as
\begin{equation}
\mathcal E
=
\left\{
(\lambda,\varphi)
\ \middle|\
30^{\circ}
\leq
|\varphi|
\leq
70^{\circ}
\right\},
\label{eq:extratropical_region}
\end{equation}
including all longitudes in both hemispheres. This fixed latitude band contains much of the midlatitude circulation in which synoptic ridges and troughs evolve. It is used as an operational extratropical diagnostic rather than as a dynamically identified atmospheric waveguide, whose position varies between cases and seasons.

RMSE and centred anomaly correlation are computed using the definitions in \Cref{app:setup}, restricted to grid points in $\mathcal E$ and retaining the cosine-latitude weights.
The comparison includes WaveGSSM, GraphCast$^{*}$, GraphSSM, and GNN-LSTM.
GraphCast$^{*}$ is the backbone-matched snapshot predictor.
GraphSSM provides a graph state-space comparison, while GNN-LSTM represents a conventional recurrent temporal model.
The latter two retain their own temporal processors and should therefore be interpreted as cross-model baselines rather than controlled architectural ablations.

Across the lead days shown in \Cref{fig:exp10_waveguide_skill_curves}, WaveGSSM achieves the lowest RMSE and highest anomaly correlation for both $z200$ and $z500$. At day~7, it reduces $z200$ RMSE by $11.2\%$ relative to GraphCast$^{*}$, $9.4\%$ relative to GraphSSM, and $4.1\%$ relative to GNN-LSTM. For $z500$, the corresponding reductions are $4.5\%$, $8.6\%$, and $3.4\%$.

The anomaly-correlation curves show a clearer separation than the RMSE curves at the later leads. This is consistent with the global result that WaveGSSM retains better agreement with the observed anomaly pattern as the rollout progresses. The comparison shows that the improvement is not restricted to one matched snapshot baseline and remains present in a region where evolving synoptic structures are prominent. It does not, by itself, identify which architectural component causes the difference between WaveGSSM and the two external temporal baselines.

\begin{figure}[th]
\centering
\includegraphics[width=\linewidth]{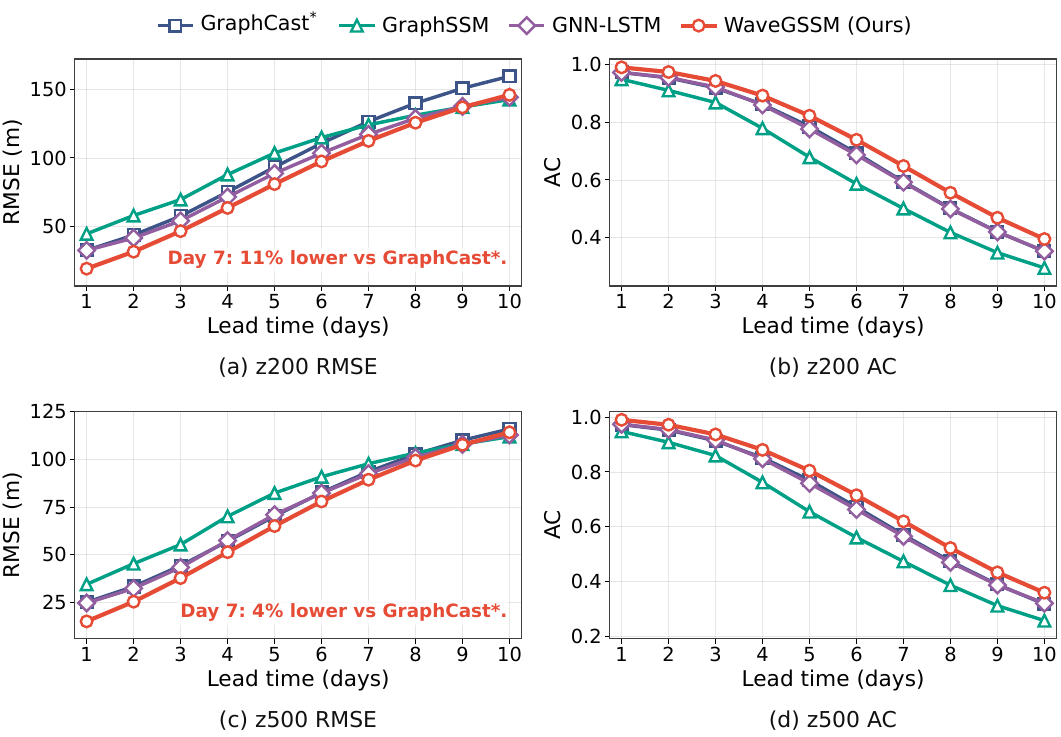}
\caption{Forecast skill across temporal model families in the extratropics. RMSE and centred anomaly correlation are computed for $z200$ and $z500$ over all longitudes between $30^{\circ}$ and $70^{\circ}$ absolute latitude. The evaluation uses the 354 test initializations from 2025 and the same cosine-latitude weighting as the global metrics. Curves compare WaveGSSM, the backbone-matched GraphCast$^{*}$ model, GraphSSM, and GNN-LSTM. WaveGSSM has the lowest RMSE and highest anomaly correlation throughout the displayed day-1 to day-7 horizon.}
\label{fig:exp10_waveguide_skill_curves}
\end{figure}

\subsection{Spatial Comparison with GraphSSM}
\label{app:graphssm_spatial_comparison}

The previous subsection shows that WaveGSSM outperforms GraphSSM in aggregate extratropical scores. We now examine whether this difference is geographically widespread or concentrated in a small number of locations.

For variable $v$, lead $h$, and grid point $x$, we define
\begin{equation}
S_{v,h}^{\mathrm{GraphSSM}}(x)
=
\frac{
\operatorname{RMSE}_{\mathrm{WaveGSSM}}(v,h,x)
-
\operatorname{RMSE}_{\mathrm{GraphSSM}}(v,h,x)
}{
\operatorname{RMSE}_{\mathrm{GraphSSM}}(v,h,x)
}.
\label{eq:local_graphssm_rmse_gap}
\end{equation}
Negative values indicate lower local RMSE for WaveGSSM.
As in
\Cref{app:spatial_skill_maps},
the local RMSE is computed over the 354 test initializations independently at every grid point.

WaveGSSM has lower latitude-weighted global RMSE than GraphSSM for both upper-air fields at all three displayed leads. The relative reductions for $z200$ and $z500$ are $31.9\%$ and $30.7\%$ at day~3, $21.3\%$ and $20.3\%$ at day~5, and $9.2\%$ and $8.3\%$ at day~7.
The reduction narrows with lead, but the spatial advantage remains broad.
At day~7, WaveGSSM has lower local RMSE at $92.8\%$ of valid $z200$ grid points and $89.4\%$ of valid $z500$ grid points.

The maps therefore show that the aggregate improvement over GraphSSM is not driven by a small number of favorable regions. The strongest margins occur at days~3 and~5, while the day-7 comparison remains positive across most grid points even as the global percentage difference becomes smaller.

GraphSSM and WaveGSSM differ in their state parameterization, processor design, and model configuration. This comparison consequently does not constitute an ablation of motion-state carry and cannot attribute the spatial difference to a single component. Its role is to show that WaveGSSM's advantage extends beyond the backbone-matched GraphCast$^{*}$ comparison and remains geographically broad against an independently designed graph state-space model.

\begin{figure}[th]
\centering
\includegraphics[width=\linewidth]{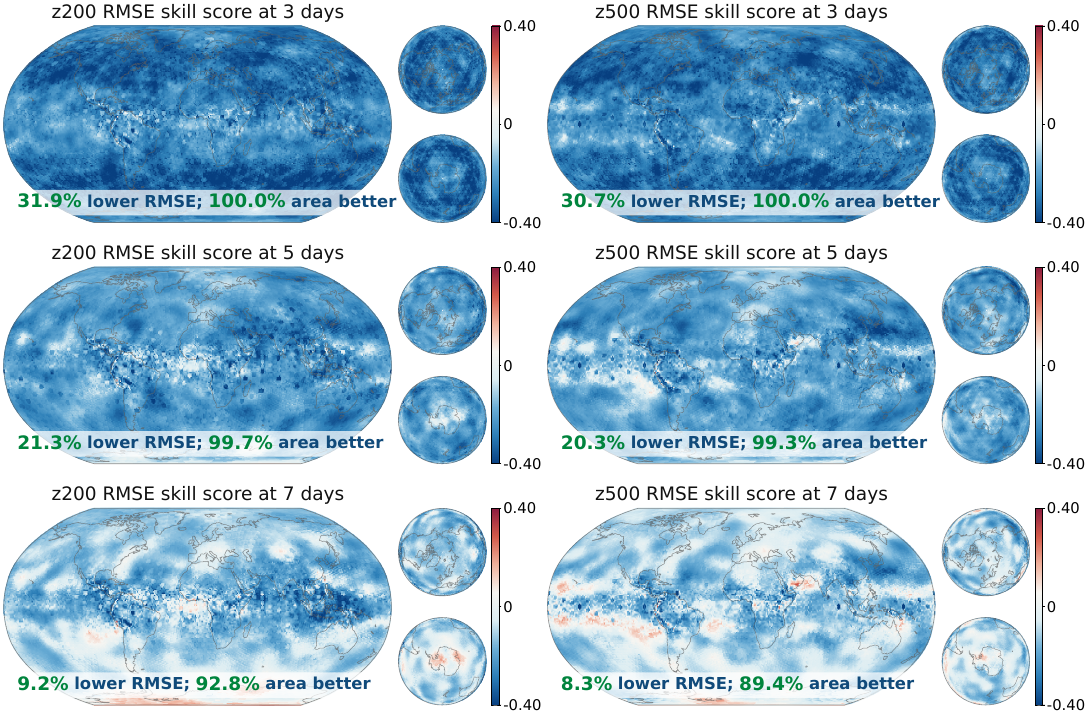}
\caption{Spatial forecast skill relative to GraphSSM. Each panel shows the normalized local RMSE difference in Eq.~(\ref{eq:local_graphssm_rmse_gap}) for $z200$ or $z500$ at lead days~3, 5, and~7. Blue indicates lower WaveGSSM RMSE, while red indicates lower GraphSSM RMSE. Polar insets display the high-latitude regions. The badges report the latitude-weighted global RMSE reduction and the fraction of valid grid points at which WaveGSSM has lower RMSE. The comparison is between two distinct graph state-space architectures and is not a component ablation.}
\label{fig:exp12_graphssm_spatial_skill}
\end{figure}

\end{document}

%% file: math_commands.tex
\usepackage{amsmath,amsfonts,bm}

\def\eqref#1{equation~\ref{#1}}

\def\1{\bm{1}}

\DeclareMathAlphabet{\mathsfit}{\encodingdefault}{\sfdefault}{m}{sl}
\SetMathAlphabet{\mathsfit}{bold}{\encodingdefault}{\sfdefault}{bx}{n}

\newcommand{\R}{\mathbb{R}}